\documentclass[sigconf]{aamas} 

\usepackage{balance} 

\setcopyright{ifaamas}
\acmConference[AAMAS '22]{Proc.\@ of the 21st International Conference
on Autonomous Agents and Multiagent Systems (AAMAS 2022)}{May 9--13, 2022}
{Online}{P.~Faliszewski, V.~Mascardi, C.~Pelachaud,
M.E.~Taylor (eds.)}
\copyrightyear{2022}
\acmYear{2022}
\acmDOI{}
\acmPrice{}
\acmISBN{}
\usepackage{microtype}
\usepackage{graphicx}
\usepackage{booktabs} 
\usepackage{url}

\newtheorem{theorem}{Theorem}

\newtheorem{lemma}[theorem]{Lemma}
\newtheorem{corollary}[theorem]{Corollary}

\newtheorem{definition}{Definition}[section]
\usepackage{algorithm}%
\usepackage{algorithmic}
\usepackage{graphicx}
\usepackage{textcomp}
\usepackage{xcolor}

\usepackage{multirow}
\usepackage{subcaption}
\usepackage[font={small}]{caption}
\let\labelindent\relax
\usepackage[inline]{enumitem}
\usepackage{amsmath}

\acmSubmissionID{350}

\title[AAMAS-2022 Formatting Instructions]{Characterizing Attacks on Deep Reinforcement Learning}

\author{Xinlei Pan$^{*}$}
\affiliation{
  \institution{UC Berkeley}
  \country{United States}
}
\email{xinleipan@berkeley.edu}

\author{Chaowei Xiao$^{*}$}
\affiliation{
  \institution{NVIDIA, ASU}
  \country{United States}
}
\email{xiaocw@asu.edu}

\author{Warren He}
\affiliation{
  \institution{UC Berkeley}
  \country{United States}
}
\email{_w@eecs.berkeley.edu}
  
\author{Shuang Yang}
\affiliation{
  \institution{Alibaba}
    \country{P.R.China}
}
\email{shuang.yang@antfin.com}

\author{Jian Peng}
\affiliation{
  \institution{UIUC}
    \country{United States}
}
\email{jianpeng@illinois.edu}
  
\author{Mingjie Sun}
\affiliation{
  \institution{Carnegie Mellon University}
  \country{United States}
 }
\email{mingjies@cs.cmu.edu}

\author{Jinfeng Yi}
\affiliation{
  \institution{JD AI Research}
  \country{P.R.China}
}
\email{yijinfeng@jd.com}

\author{Zijiang Yang}
\affiliation{
  \institution{Xi'an Jiao Tong University}
    \country{P.R.China}
}
\email{zijiang@xjtu.edu.cn}

\author{Mingyan Liu}
\affiliation{
  \institution{University of Michigan, Ann Arbor}
  \country{United States}
}
\email{mingyan@umich.edu}

\author{Bo Li}
\affiliation{
  \institution{UIUC}
  \country{United States}
}
\email{lbo@illinois.edu}
  
\author{Dawn Song}
\affiliation{
  \institution{UC Berkeley}
  \country{United States}
}
\email{dawnsong@cs.berkeley.edu}
\thanks{$^*$indicates equal contribution}

\begin{abstract}
Recent studies show that Deep Reinforcement Learning (DRL) models are vulnerable to adversarial attacks, which attack DRL models by adding small perturbations to the observations. However, some attacks assume full availability of the victim model, and some require a huge amount of computation, making them less feasible for real world applications. In this work, we make further explorations of the vulnerabilities of DRL by studying other aspects of attacks on DRL using realistic and efficient attacks. First, we adapt and propose efficient black-box attacks when we do not have access to DRL model parameters. Second, to address the high computational demands of existing attacks, we introduce efficient online sequential attacks that exploit temporal consistency across consecutive steps. Third, we explore the possibility of an attacker perturbing other aspects in the DRL setting, such as the environment dynamics. Finally, to account for imperfections in how an attacker would inject perturbations in the physical world, we devise a method for generating a robust physical perturbations to be printed. The attack is evaluated on a real-world robot under various conditions. We conduct extensive experiments both in simulation such as Atari games, robotics and autonomous driving, and on real-world robotics, to compare the effectiveness of the proposed attacks with baseline approaches. To the best of our knowledge, we are the first to apply adversarial attacks on DRL systems to physical robots.
\end{abstract}
\newcommand{\BibTeX}{\rm B\kern-.05em{\sc i\kern-.025em b}\kern-.08em\TeX}

\begin{document}
\pagestyle{fancy}
\fancyhead{}

\maketitle 

\begin{figure}[t!]
  \centering
    \includegraphics[width=\linewidth]{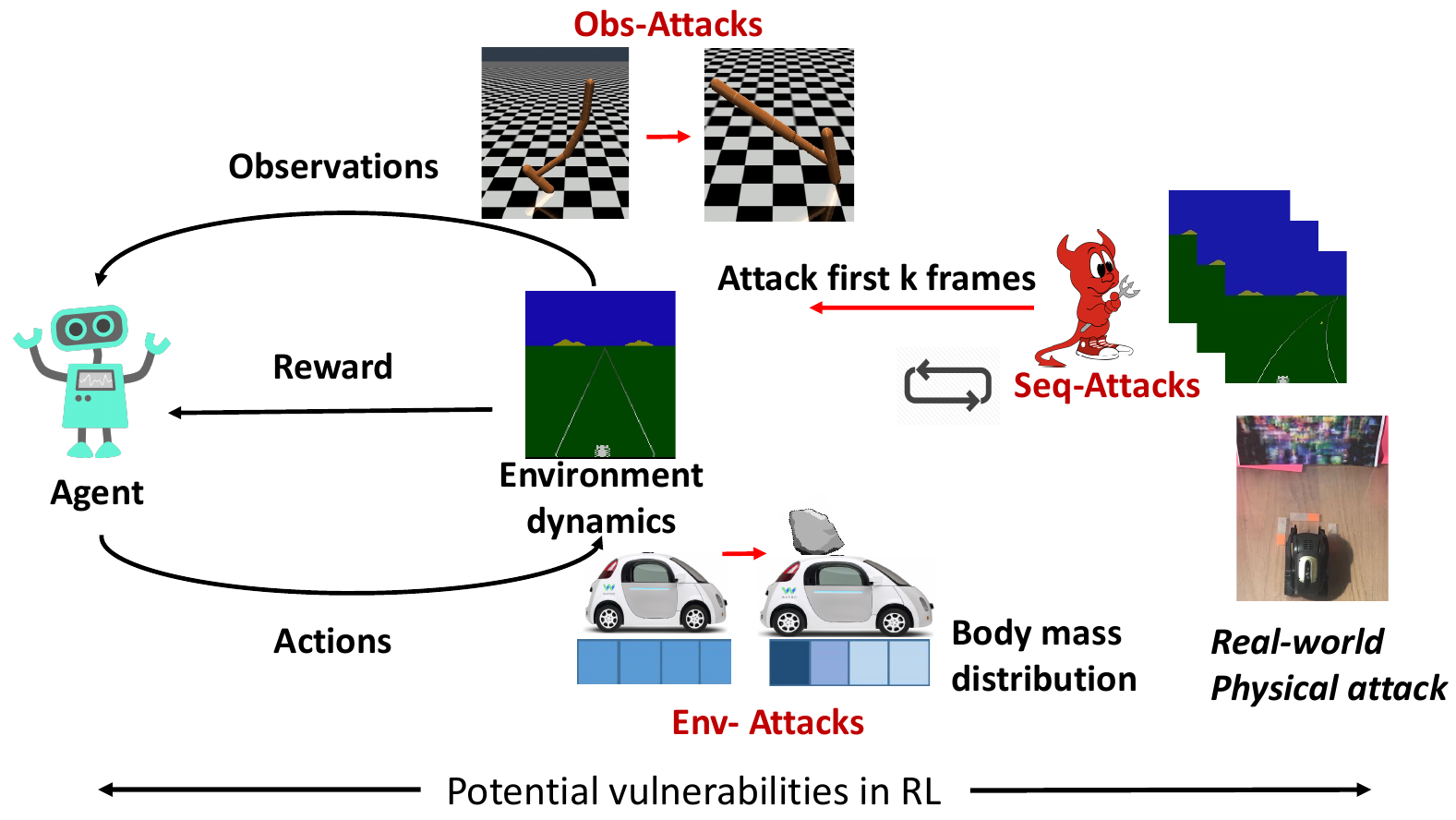}
    \caption{Analysis of adversarial attacks on DRL. RL environments are usually modeled as a Markov Decision Process (MDP), and its \textit{observation} space, and \textit{environment (transition) dynamics} may be attacked. We propose adversarial attacks that have improved computational efficiency by using the sequential nature of MDP, with attacks in both digital and physical environments.}
    \label{fig:taxonomy}
\end{figure}

\section{Introduction}
With recent progress of DRL in various applications, such as computer games~\cite{ghory2004reinforcement,mnih2013playing,mnih2016asynchronous},
autonomous navigation~\cite{dai2005approach,pan2019spc} and robotics~\cite{levine2015end}, the safety and robustness of DRL models
are becoming a major concern, especially on real world robotics tasks~\cite{levine2015end,tan2021systems,pan2020zero}. Recently, adversarial attacks that have imposed challenges to general deep neural network (DNN) models~\cite{goodfellow2014explaining,li2014feature,xiao2019meshadv,xiao2018spatially,xiao2018generating,qiu2019semanticadv,xiao2018characterizing} started to challenge the robustness of DRL models. However, the adversarial attacks on DRL models can be different from attacks on DNN models: DRL models focus on sequential decision-making problems while DNN models mostly work on individual prediction problems with no temporal continuity. Existing white-box attacks on DRL models assume almost full access to the victim policy~\cite{huang2017adversarial}. Some black-box attacks assume partial knowledge of the victim policy~\cite{huang2017adversarial,gleave2019adversarial} but attack each observation individually. These black-box attacks can be computationally intensive especially on tasks with high dimensional inputs. Moreover, most of these attacks have only been evaluated in simulated environments. It remains unclear of their effects on real world DRL models.

In this work, we seek to expand our understanding of the vulnerabilities of DRL systems. To achieve this goal, we propose a set of realistic attacks with improvement on the computational efficiency. To make less assumptions on the victim model, we focus on black-box attack. Based on the components of MDP (shown in Figure~\ref{fig:taxonomy}), we categorize attacks into different types, including attacks on observation space and on environment dynamics. To improve the attack efficiency on multiple high dimensional inputs, we propose to reduce the amount of computation by utilizing the sequential nature of MDP tasks. Finally, to validate the feasibility of adversarial attacks on real world DRL tasks, we perform real world physical attack on a real robot control task. We summarize the proposed attacks on DRL models and our contribution as follows.

\textbf{Advanced black-box attacks.}
An attacker may not have access to the DRL system's internals,  making white-box attacks infeasible in this case.
For DNN models, black-box attacks that take advantage of query access to the victim model have emerged~\cite{bhagoji2017exploring}. Based on this progress, we adapt and improve this method for attacking DRL systems and demonstrate the effectiveness. Specifically, starting from the finite difference (FD) based attack~\cite{bhagoji2017exploring}, we propose an improvement, named {\textit{adaptive sampling FD (SFD)}},  that reduces the amount of computation by adaptively sampling the input dimensions for gradient estimation.
We provide a theoretic analysis of SFD method and prove its efficiency and estimation error bound. 

\textbf{High throughput perturbation.}
The adaptive SFD method is still inefficient since the attacker would need to generate adversarial examples on each individual frame. To make further improvement, we propose an {\textit{online sequential FD attack}} based on the fact that DRL tasks have temporal continuity, where temporal consecutive frames tend to correlate with each other. We hypothesize that the attacks generated on a small group of \textit{selected frames} can be applied globally to other similar and temporally close frames. When limiting the number of frames to be selected, it is important to find the best set of frames for generating such attacks. We observe that not all action decisions (frames) are critical within a trajectory, and hypothesize that attacks that are generated based on a small group of important and critical frames can be more effective. Based on this, we propose an {\textit{optimal frame selection}} approach to select the most important frames to generate the attack. We provide mathematical analysis of this approach. 

\textbf{Perturbations on components other than the observation.}
As shown in Figure~\ref{fig:taxonomy}, the observation space is not the only component in DRL systems. Besides,  environment dynamic is another  important component in the DRL systems. 
Thus, we propose another attacks that perturb the environment transition dynamics by changing physical properties of the environment rather than changing the input observation to the victim model.

\textbf{Physical attacks.}
While we have shown the digital adversarial examples on DRL systems, it is still unclear about the possibility  to generate   physical adversarial examples to attack the physical DRL system.
We bridge this gap by using a printed adversarial patch and a toy robot visual navigation task.
We show that the overall end-to-end attack is effective even under effects such as manufacturing inaccuracy and varied viewing angles.

We conduct extensive experiments on the above proposed attacks and compare them with existing white-box (which could be viewed as the performance upper bound) and black-box attacks both in simulation and on a real robot. We show that it is feasible to attack real-world DRL systems with our proposed approaches.

\begin{table*}
    \centering
    \caption{Summary of the adversarial attacks on DRL systems,
    categorized based on our proposal.
    The name reflects the category of the attack method. For example, \textsf{obs-fgsm-wb} means attack
    on observation using fast gradient sign method based white-box attack and \textsf{obs-fgsm-bb} means attack on observation using fast gradient sign method based black-box attack. The attack methods we proposed are
    highlighted using bold text. ``Arch.,'' ``Param.,'' and ``Query'' indicate whether the attack requires knowledge of the policy network's architecture, parameters and whether it needs to query the policy network.}
    \label{tab:attacks}
    \vspace{-10pt}
    \resizebox{\linewidth}{!}{%
    \begin{tabular}{llcccccccc}
        \toprule
        Attack             &  MDP Component & \multicolumn{4}{c}{Attacker Knowledge} & Real-time & Physical & Temporal Dependency \\
        \cmidrule(r){1-1} \cmidrule(rl){2-2} \cmidrule(lr){3-6}   \cmidrule(l){7-9} 
        & & White/Black-Box & Arch. & Param. & Query &  & & \\
        \cmidrule(lr){3-6}
        \textsf{obs-fgsm-wb}       & Observation   & White-box     & Yes & Yes & Yes & Yes      & No  & Independent \\ 
        \textsf{obs-fgsm-bb}      & Observation  & 
        Black-box & No & No & No &  Yes & No  & Independent \\ 
        \textbf{\textsf{obs-fd-bb}}       & Observation   & 
        Black-box & No & No & Yes & Too slow & No  & Independent \\
        \textbf{\textsf{obs-sfd-bb}}       & Observation   &  Black-box & No & No & Yes & slow & No  & Independent \\
        \cmidrule(lr){3-6}
        \textbf{\textsf{obs-seq-fgsm-wb}}   & Observation   & White-box       & Yes & Yes & Yes & Yes      & No  & Sequential  \\
        \textbf{\textsf{obs-seq-fd-bb}}     & Observation   & 
        Black-box & No & No & Yes & Yes      & No  & Sequential  \\
        \textbf{\textsf{obs-seq-sfd-bb}}    & Observation   & 
        Black-box & No & No & Yes & Yes      & No  & Sequential  \\
        \cmidrule(l){2-9}
        \textbf{\textsf{env-search-bb}}     & Transition Dynamics & Black-box     & No & No & Yes & N/A  & Yes & N/A         \\
        \bottomrule
    \end{tabular}}
    \vspace{-10pt}
\end{table*}

\section{Related work}

\label{sec:relatedwork}

\textbf{Adversarial attacks on machine learning models}.
Our attacks draw inspirations from previously proposed attacks.
~\citeauthor{goodfellow2014explaining}~\cite{goodfellow2014explaining} describes the fast gradient sign method (FGSM) of generating adversarial perturbations in a white-box setting.
~\citeauthor{carlini2016towards}~\cite{carlini2016towards} describe additional methods based on optimization, which results in smaller perturbations.
~\citeauthor{moosavi2016universal}~\cite{moosavi2016universal} demonstrates a way to generate a ``universal'' perturbation that is effective on multiple inputs. ~\citeauthor{xiao2019meshadv}~\cite{xiao2019meshadv} generates adversarial examples in 3D world by changing the shape and texture information respectively. 
~\citeauthor{evtimov2017robust}~\cite{evtimov2017robust} shows that adversarial examples can be robust to natural lighting conditions and viewing angles using real world examples.
Furthermore, black-box attacks without providing victim model's training algorithms are also proposed for general machine learning models~\cite{papernot2017practical,chen2017zoo}.

\textbf{Adversarial attacks on DRL}.
DRL methods train a policy that maps state observations to action decisions. Examples include Deep Q-Learning (DQN)~\cite{mnih2015human} for discrete control, and Deep Deterministic Policy Gradient (DDPG)~\cite{lillicrap2015continuous} for continuous control. Proximal Policy Optimization (PPO)~\cite{schulman2017proximal}, and Soft Actor-Critic (SAC)~\cite{haarnoja2018soft} are also proposed recently.
We select DQN and DDPG as our target victim algorithms, but our attacks can apply to other RL methods.

Recently, ~\citeauthor{huang2017adversarial} demonstrates an attack that uses FGSM to perturb observation frames in a DRL setting~\cite{huang2017adversarial}. However, the white-box setting in this work requires knowing the full victim model and the preferred action. They also propose a black-box attack method based on transferability, where the surrogate models are trained to obtain attacks and the generated attacks are then applied on the victim models. We build on FD-based attacks that do not rely on transferability. ~
\citeauthor{lin2017tactics}~\cite{lin2017tactics} designs
an algorithm to achieve targeted attack for DRL models, and they propose a method to select optimal frames for attacks based on the preference of the policy on the best action over the worst action. We provide related theoretical proofs to demonstrate the optimality of frame selections. ~\citeauthor{behzadan2017vulnerability}~\cite{behzadan2017vulnerability} propose a black-box attack method that trains another DQN network
to minimize the expected return using FGSM. ~\citeauthor{gleave2019adversarial}~\cite{gleave2019adversarial} proposes to train another agent to interact and modify the environment so as to indirectly attack the victim model. The black-box attacks in these related works are all evaluated on simulated environments. We provide real world experiments validating the effectiveness of our proposed attacks. For adversarial
attacks on environment dynamics, ~\citeauthor{pan2019you}~\cite{pan2019you} propose
to use candidate inference attack to infer possible dynamics used for training a candidate policy, posing potential privacy-related risk
to deep RL models. 

\textbf{Robust RL via adversarial training}.
Safety and generalization in various robotics and autonomous driving applications
have drawn significant attention for training robust models~\cite{packer2018assessing,pinto2017robust,pan2019risk}. Knowing 
how DRL models can be attacked
is beneficial for training robust DRL agent.
~\citeauthor{pinto2017robust} proposes to train 
a RL agent to provide adversarial attack during training 
so that the agent can be robust against
dynamics variations~\cite{pinto2017robust}. 
However, since they manually selected the perturbations on
environment dynamics, the attack provided in their work may not be able to generalize to broader RL systems.
Additionally, their method relies on an accurate modeling of the environment
dynamics, which may not be available for real world tasks such as robotics systems.
\vspace{-5pt}

\section{ Threat Model On DRL}
\label{sec:taxonomy}
Our victim models are trained by interacting with environments that are Markov Decision Processes (MDPs), which include several components: the state space $\mathcal{S}$, the action space $\mathcal{A}$, the transition dynamics $\mathcal{T}$ and the reward function $\mathcal{R}$. The goal of the DRL model is to learn a policy $\pi$ so as to maximize the agent's future expected return $\mathbb{E}_{\pi}[\sum_t \gamma^tr_t]$, where $\gamma$ is a discount factor. In this work, we provide methods for adversarial attack for trained DRL policies, including discrete control and continuous control methods. We select two representative algorithms: DQN~\cite{mnih2015human} for discrete control and DDPG~\cite{lillicrap2015continuous} for continuous control.

We aim to attack well-trained DRL models without accessing the victim model's parameters and only querying the victim model to get model output. The goal of the attacker is to minimize the agent's future expected return.  We do not assume these attackers have full control over the agent nor the robotics system. They are weaker than, for example, an attacker that could generally compromise the robot's software. To avoid trivial detection, an adversary needs to constrain the magnitude of perturbation. In this work, 
we bound the $L_{\inf}$ norm of the added perturbation during the evaluation of the digital attacks. 
For real world physical attacks, we follow the common settings in the literature~\cite{evtimov2017robust,athalye2017synthesizing}.

\section{Adversarial Attacks on DRL} 
\label{sec:attacks}

In this section we develop several concrete attacks to improve the attack feasibility and efficiency. We first introduce some baseline attacks and then describe our new attacks in detail. 
Table~\ref{tab:attacks} summarizes these attacks, where we categorize them based on their attack components (attack observation or transition dynamics), attacker's knowledge, the computational efficiency of the attack (real-time), whether the attack requires physically changing the environment (physical), and whether the attack is based on temporally dependency of consecutive frames (independent or sequential). 
\vspace{-0.3cm}

\subsection{Baseline Attacks}
We discuss both white-box and black-box baseline attack methods. 

\textbf{White-box attacks}.
In this setting, we assume that the attacker can access the agent's policy 
network $\pi(a|s)$ where $a$ refers to the action and $s$ refers to 
the state.
\citeauthor{huang2017adversarial}\cite{huang2017adversarial} has previously introduced one attack in this category that applies the FGSM method to generate white-box perturbation purely 
on observations.
We reproduce this experiment with our \textbf{\textsf{obs-fgsm-wb}} attack. This attack's application scenario is
when we know the policy network's architecture and parameters. 

\textbf{Black-box attacks}.
There are different scenarios depending on the attacker's knowledge. In one scenario, the attacker does not have any information about the model architecture, or parameters, and it can't query the model either. In this case, the attacker can perform a ``transferability" based attack by attacking a surrogate model (of which it has complete knowledge) and then transfer the perturbation to the victim model.
\citeauthor{huang2017adversarial}\cite{huang2017adversarial} introduced a black-box variant of the FGSM attack using transferability, which we denote as \textbf{\textsf{obs-fgsm-bb}}.
Their experiments assume that the attacker has access to the same environment, knows the training algorithm of the targeted agent, and uses the same algorithm to train the victim model.
\vspace{-0.3cm}

\subsection{Advanced Black-Box Attacks}
In an alternative black-box setting, the attacker has access to query the model, obtaining the model's outputs on given inputs (while still not knowing the model architecture or parameters).
This setting represents a realistic real-world deployment, for example where an executable copy of the agent is shipped to customers, or where an online service is offered. In this setting, we propose more advanced black-box attacks, which take advantage of this query access.

\textbf{Black-box finite difference (FD) based attack}.
Baseline black-box attack \textbf{\textsf{obs-fgsm-bb}} requires retraining a surrogate policy. Previous work~\cite{bhagoji2017exploring} applies the finite difference (FD) method in attacking classification models.
We extend the FD method to DRL systems in \textbf{\textsf{obs-fd-bb}} which doesn't require retraining a new policy. This attack works in
the setting where we don't have the policy network's
architecture or parameters and the training algorithms, but can query the network.
FD based attack on DRL uses FD to estimate gradient with respect to 
the input observations. It then generates perturbations on the input observations by using the estimated gradient.
The key step in FD is to estimate the gradient. Denote the loss function as $L$ and state input
as $\mathbf{s}\in\mathbb{R}^d$. The
canonical basis vector
$\mathbf{e}_i$ is defined as an $d$ dimension vector with 1 only in the $i$-th
component and 0 otherwise. The FD method estimates gradients
via the following equation
\begin{equation}
\begin{split}
    \nabla_{\mathbf{s}}L(\mathbf{s})=\mathbf{FD}(L(\mathbf{s}), \delta) = & \bigg[\frac{L(\mathbf{s}+\delta\mathbf{e}_1)-L(\mathbf{s}-\delta\mathbf{e}_1)}{2\delta}, \\
    & \cdots, \frac{L(\mathbf{s}+\delta\mathbf{e}_d)-L(\mathbf{s}-\delta\mathbf{e}_d)}{2\delta}\bigg]^\intercal,
    \end{split}
\end{equation}
where $\delta$ is a parameter to control estimation accuracy. The loss function $L$ depends on the actual RL algorithm we use, and since we have the model output, we can select an action with a minimal value. Define the state-action value function as $Q(s,a)$, then given a state $s$, we can obtain a bad target action $a_{t}$ from the model output as
$a_t = \arg\min_{a}Q(s,a)$.
Then we define $L$ to induce the model to select that bad action as follows. Denote the actor of the RL algorithm as $\pi$; the loss function for continuous control is:
    $L(\mathbf{s}) = \|\pi(\mathbf{s})-a_t\|_2^2,$
and the loss function for discrete control is,
$    L(\mathbf{s}) = \text{CELoss}(\pi(\mathbf{s}), a_t)$,
where CELoss() is the cross-entropy loss for classification. In the DQN setting, there is no actor, but we can define an action probability distribution as $$
    \pi(\mathbf{s}) = \arg\max_{a} \frac{\exp(Q(\mathbf{s},a))}{\sum_{a_i\in\mathcal{A}}\exp(Q(\mathbf{s},a_i))}.$$
For $d$ dimensional input, the finite difference method
would require $2d$ queries to obtain the estimation, which is 
computationally intensive for high dimensional inputs such as images .
Therefore, we propose a sampling technique to mitigate this computational cost.

\textbf{Adaptive sampling based FD (SFD)}. Many deep learning models extract features
from inputs patch-wise and have sparse activation map~\cite{bau2017network}.
Based on the this observation, we propose a method for estimating gradients that exploits this spatial structure.
In this method, we first estimate the gradient with respect to some randomly sampled pixels, then iteratively, we identify pixels where the gradient has a high magnitude and estimate the gradient with respect to surrounding pixels.

Given a function $f(\cdot;w): \mathbb{R}^d\rightarrow \mathbb{R}^1$, where
$w$ is the model parameter (we omit this for conciseness below),
our goal is to estimate the gradient of $f$
with respect to an input $x \in  \mathbb{R}^d$: $\nabla_x\hat{f}(x)$.
We define the nontrivial dimension of the gradient of $f$ at $x$ as $\{j\in\{1,2,\cdots,d\}; |\nabla_jf(x)|\geq\theta\}$,
i.e., the dimensions with gradient absolute value greater or equal to a threshold value $\theta>0$.
To estimate nontrivial dimension of the gradient, first,
we randomly sample $k$ dimensions from $\{1, \cdots, d\}$,
to get a set of dimensions $S=\{S_1, S_2, \cdots, S_k\}$,
and use FD to estimate the gradients
for dimensions in $S$. Then we select a
set of dimensions $S'=\{j\in S;|\nabla_jf(X;w)|\geq\theta\}$, and use FD to estimate the gradients of the neighbors (a set $S''$)
of dimensions in $S'$, if these gradients haven't been estimated (for dimension $i$ within a $d$ dimensional vector, the neighbor dimension is defined as dimension $i+1$ if it exists). 
Then we select dimensions with absolute gradients no less than $\theta$ from $S''$ and find their neighbors to estimate gradients. We repeat this process for multiple iterations. 
By exploring the sparse large gradients this way, we can adaptively
sample dimensions to estimate gradients, which can significantly
reduce the number of queries. We give the full attack algorithm of \textbf{\textsf{obs-sfd-bb}}.
in our appendix. We denote \textbf{\textsf{obs-s[n]fd-bb}} as the attack of \textsf{obs-sfd-bb} with $n$ iterations.

To better understand the benefits of SFD, here we provide an analysis of this
algorithm and estimate the amount of nontrivial dimension of the gradient
that can be estimated using our method in
Lemma~\ref{co2}. The basic idea of this lemma is to prove that by using SFD, we can sample more of the nontrivial dimensions of the gradient than by using random sampling. We also provide an error bound for the estimated gradient with SFD in Theorem~\ref{theorem2}.

\theoremstyle{definition}
\begin{definition}[Neighbor Dimension's Gradient]
$\forall i,j\in\{1,2,\cdots,d\}$ and $j=i+1$, we define the neighbor dimension's gradient as $\nabla_if(x)^N = \nabla_{j=i+1}f(x)$. 
Note that $j=i+1$ is equivalent to $j=i-1$, and to be general we choose 
the first one to obtain the definition. 
\end{definition}

\theoremstyle{definition}
\begin{definition}[Non-trivial Gradient Dimension]
Given a positive gradient threshold $\theta$, an input data instance $x\in\mathbb{R}^d$, and a loss function $f: \mathbb{R}^d\rightarrow\mathbb{R}^1$, for any dimension $i\in\{1,\cdots, d\}$,
if $|\nabla_if(x)|\geq\theta$, then we define this gradient as non-trivial gradient and the corresponding
dimension $i$ as non-trivial gradient dimension. On the other hand, if $|\nabla_if(x)|<\theta$,
then we define this gradient as trivial gradient and the corresponding dimension $i$ as trivial
gradient dimension.
\end{definition}

\theoremstyle{definition}
\begin{definition}[Gradient Sampling Probability]
Given a selected threshold $\theta>0$ in Algorithm SFD, for any $x\in\mathbb{R}^d$, define the non-trivial gradient sampling probability as,
\begin{equation}
P_A(\theta) = \frac{1}{|D_A|}\sum_{i \in D_A}1(|\nabla_if(x)|\geq\theta),
\end{equation}
where $D_A$ represents the set of dimensions selected by algorithm $A$. Therefore, the gradient sampling probability of SFD and random sampling are $P_\textit{SFD}(\theta)$ and $P_\textit{random}(\theta)$ respectively. Some further definitions on neighbor gradient distribution probability are as following:

\begin{itemize}
    \item If $|\nabla_if(x)|\geq\beta+\theta$, then define
    \begin{equation}
        \begin{split}
            p(|\nabla_if(x)^N|\in[\theta, \beta+\theta]) & = q \\
            p(|\nabla_if(x)^N|\in[\beta+\theta,\infty)) & = 1-q. \\
        \end{split}
    \end{equation}
    \item If $|\nabla_if(x)|\in[\theta, \beta+\theta]$, then define
    \begin{equation}
        \begin{split}
            p(|\nabla_if(x)^N|\in[0, \theta]) & = p_1 \\
            p(|\nabla_if(x)^N|\in[\theta, \beta+\theta]) & = p_2 \\
        p(|\nabla_if(x)^N|\in[\beta+\theta, \infty)) & = p_3.
        \end{split}
    \end{equation}
\end{itemize}
\end{definition}

Based on the above assumption that these distribution $p_1,p_2,p_3$ and $q$ are defined over
all possible dimensions in one image (over $i$) and these distribution works throughout the entire gradient estimation iteration process, we have the following lemma. 

\begin{lemma}
\label{co2}
We make
the following assumptions on $f$:  $\exists\beta>0$, s.t.
$|\nabla_if(x)-\nabla_if(x)^N|\leq\beta, \forall i\in\{1,\cdots, d-1\}, \forall x\in\mathbb{R}^d$. 
For dimension $i$ whose gradient $|\nabla_if(x)|\in [\theta, \beta+\theta]$, the probability that the gradient magnitude of its neighborhood pixel is in $[0, \theta]$ is $p_1$. We conclude,
as long as $p_1 < 1 - P_\textit{random}(\theta)$, 
we have $P_\textit{SFD}(\theta)>P_\textit{random}(\theta)$.
\end{lemma}
The intuitive understanding of this lemma is that 
when nontrivial gradients (of magnitude no less than $\theta$) dimensions are spatially concentrated ($p_1$ is small, then $p_1 < 1 - P_\textit{random}(\theta)$), our method will be more sample efficient than random sample method. We provide the proof of this Lemma in our appendix. Next we give another theorem about the upper bound for the gradient estimation error and include the proof for this theorem in our appendix.

\begin{theorem}
\label{theorem2}
Suppose we sample all nontrivial dimensions of the gradient and estimate the gradient with perturbation strength $\delta$, the estimation error of the
 gradients is upper bounded by the following inequality,
 \begin{equation}
     \|\nabla\hat{f}(x)-\nabla f(x)\|_{1} \leq S_{\theta}C{\delta}^2+(d-S_{\theta})\theta,
 \end{equation}
for constant $C>0$, $S_\theta= \sum_{i=1}^d1(|\nabla_if(x)|\geq\theta)$, and $\nabla\hat{f}(x)$ is the estimated gradient of $f$ with respect
to $x$. 
\end{theorem}

\subsection{High Throughput Attacks}
\sloppy

A DRL system operates on a sequence of consecutive frames. To develop a useful attack method against real-time DRL systems, we must consider the computational costs. Therefore, in this section, we propose a method for generating perturbations efficiently: an online sequential attack.

In a DRL setting, consecutive observations are not
i.i.d.---instead, they are highly correlated and sometimes the consecutive observations do not change too much.
It's then possible to perform an attack with less computation than performing the attack independently on each state.
Considering real-world cases, for example, an autonomous robot would take a real-time video as input to help make decisions, an attacker is motivated to generate perturbations only based on previous states and apply it to future states, which we refer to as an \textit{online sequential attack}.
We hypothesize that a perturbation generated this way is effective on subsequent states.

Therefore, we propose online sequential attacks \textbf{\textsf{obs-seq-fgsm-wb}} in whitebox setting and \textbf{\textsf{obs-seq-fd-bb}}, \textbf{\textsf{obs-seq-sfd-bb}} in blackbox setting.
In these attacks, we first collect $k$ observation frames and generate a single perturbation using the averaged gradient on these frames (or estimated gradients using FD or SFD, in the case of \textbf{\textsf{obs-seq-fd-bb}} and \textbf{\textsf{obs-seq-sfd-bb}}).
Then, we apply that perturbation to all subsequent frames. We denote them as \textbf{\textsf{obs-seq[Fk]-fgsm-wb}, \textsf{obs-seq[Fk]-fd-wb},\textsf{obs-seq[Fk]-sfd-wb}}.

Next, instead of using all the $k$ frames, we further improve the above attack by finding the the set of frames that are important and using the gradients from those frames to perform attack.
With this, we hope to maintain attack effectiveness while reducing the number of queries needed.
We propose to select a subset of frames within the first
$k$ frames based on the variance of their $Q$ values.
Then, in all subsequent frames, the attack applies a perturbation generated from
the averaged gradient. 
We select an optimal set of important frames with highest value
variance to generate the perturbations. We denote them as \textbf{\textsf{obs-seq[Lk]-fgsm-wb}, \textsf{obs-seq[Lk]-sfd-wb}, \textsf{obs-seq[Lk]-sfd-wb}}.
We give a proof in Corollary~\ref{co1} below for why attacking these important frames is more effective to reduce the overall expected return. We include the proof in our appendix.

\begin{corollary}
\label{co1}
Let the state and state-action value be $V(s)$ and $
Q(s,a)$ respectively for a policy $\pi$ with time horizon $H$. We conclude that $\forall t_1,t_2\in{1,2,\cdots, H}$, if $\textrm{Var}(Q(s_{t_1},\cdot))\geq \textrm{Var}(Q(s_{t_2},\cdot))$,
then
$
    \mathbb{E}_{\pi}\left[\sum_{t=0}^H\gamma^tr_t|do(s_{t_1}=\hat{s}_{t_1})\right]\le
    \mathbb{E}_{\pi}\left[\sum_{t=0}^H\gamma^tr_t|do(s_{t_2}=\hat{s}_{t_2})\right],
$
where $do(s_{t_1}=\hat{s}_{t_1})$ means the observation at time $t_1$ is changed from $s_{t_1}$ to $\hat{s}_{t_1}$.
\end{corollary}
\vspace{-10pt}

\subsection{Attacks on Components Other than the Observation}
Besides attacking the observation space, the attacker can have access to the testing environments 
of the victim model. Therefore, potentially the attacker can modify the environment such as changing
the physical properties of the environment to perform the attack. In this case, the observations
of the victim model are not perturbed but the environment transition dynamics will be perturbed.

\textbf{RL based attacks on environment dynamics}.
In this case, the attacker will perturb the environment transition model (dynamics). Such transition model is usually non-differentiable with respect to the policy. Therefore, we propose a novel reinforcement learning based method to attack environment dynamics. We describe a
targeted attack (in which the agent will fail in a specific way, e.g. a Hopper turn over, or a self driving car drive off the road and hit obstacles.) where the attacker changes the environment dynamics (e.g. by changing the mass of the car). The algorithm is as follows. 

Define the environment dynamics as $\mathcal{M}$, the agent's policy
as $\pi$, the agent's state at step $t$ following the current policy
under current dynamics as $s_t$, and define a mapping
from $\pi,\mathcal{M}$ to $s_t$:$s_t \sim f(s_t|\pi,\mathcal{M}, s_0)$,
which outputs the state at time step $t$: $s_t$ given initial
state $s_0$, policy $\pi$, and environment dynamics $\mathcal{M}$. 
The task of attacking environment dynamics is to find another
dynamics $\mathcal{M'}$ such that the agent will reach a target
state $s_t'$ at step $t$: $\mathcal{M'} = \arg\min_{\mathcal{M}}\|s_t'-\mathbb{E}_{s_t\sim f(s_t|\pi, \mathcal{M}, s_0)}[s_t]\|$.

We consider the following two algorithms for generating this attack:
First, \textbf{Random dynamics search}.
A naive way to find the target dynamics, which we demonstrate in \textbf{\textsf{env-rand-bb}}, is to use random search. Specifically, we randomly propose a new dynamics and see whether,
under this dynamics, the agent will reach $s_t'$. This
method works in the setting where we don't need to 
have access to the policy network's architecture and parameters, but just need to query the network. Second, \textbf{RL based adversarial dynamics search}. We design a more systematic
algorithm based on RL to search for a dynamics to attack and call this method \textbf{\textsf{env-search-bb}}. The algorithm is included in appendix. At each time step, an attacker proposes a change to the current environment dynamics with some perturbation $\Delta\mathcal{M}$, where $\|\Delta\mathcal{M}/\mathcal{M}\|$ is 
bounded by some constant $\epsilon$,
and we find the new state $s_{t,\mathcal{M'}}$ at time step $t$ following the current policy under dynamics $\mathcal{M'}=\mathcal{M}+\Delta\mathcal{M}$, then the
attacker agent will get reward $\tilde{r}=1/\|s_{t,\mathcal{M'}}-s_t'\|$. 
We demonstrate this in \textsf{env-search-bb} using DDPG \cite{lillicrap2015continuous} to train the attacker. In order
to show that this method works better than random search, we also
compare with the random dynamics search method, and keep 
the bound of maximum perturbation $\|\Delta\mathcal{M}/\mathcal{M}\|$
the same.

\subsection{ Physical Attacks}
\label{sec:physattack}
\sloppy
We discuss how to apply previous proposed FD attack algorithms to generate adversarial perturbations on real world DRL models. We choose visual navigation robots as our experiment platform. There are several challenges to perform real-world attacks in this task. 
(1) Imperfect camera. Images observed by the robot usually are captured by cameras, and the computed perturbation may not be directly applicable on real world objects due to camera sensor noise and color shift.
(2) Imperfect fabrication process. Using a printer to make the adversarial image patch limits the attacker to printable colors, and there exists a color gap between input image to the printer and output paper from the printer. 
(3) Imperfect alignment of the patch. 
Mounting the adversarial image patch at an exact position with a particular orientation is hard. Thus, the attack should be robust to variations of relative position/orientation of the image patch to the robot position. 
In order to overcome the above challenges, we adapt our algorithm to generate an adversarial patch that is robust against these imperfections. We select a discrete control task for visual navigation as our real world victim model (see details in experiment section).

In order to improve the robustness of the generated patch against various mounting positions, we randomly sample multiple binary masks $\{K_i\}_{i=1}^{n_{1}}$ and apply the masks to the state (an image) $I$, such that any part of the patch can be used for attack. One mask is consisted of a rectangular region with value 1 and all other regions with value 0. Denote the generated perturbed image as $I'$, then the masked frame $I^\textit{masked}_i$ with mask $K_i$ could be defined as:
  $  I^\textit{masked}_i = I \odot (1-K_i) + I' \odot K_i,$
where $\odot$ is the element-wise product. 
To further improve robustness against imperfect printer and variability of mounting orientation, we randomly generate a set of transformations $T=\{T_j\}_{j=1}^{n_2}$ including contrast, brightness adjustments, random rotation adjustments  on $I^\textit{masked}$ and get the final image to be optimized: $I_\textit{final}=T(I^\textit{masked})$. 
Define the function of generating the final image $I_\textit{final}$ as $u$: $I^\textit{final}_{ij} = T(I^\textit{masked}_i)_j =  u(I'; I, K_i, T_j)$. 
Given a trained policy $\pi$ with parameters $\theta$, the pristine optimal action output is $a^* = \pi_{\theta}(I)$. We select a target action $a'$, which should have smaller return than $a^*$. The objective function is shown as follows:
$ I' = \arg\min_{I'}\sum_{i=1,j=1}^{i=n_1,j=n_2} CE(\pi_{\theta}(u(I'; I, K_i, T_j )), a'),
$ where $CE$ denotes the cross entropy loss. We then apply online-sequential method and use sampling based FD-based method \textsf{obs-seq-sfd-bb} to estimate the perturbation.

\begin{table*}[h!]
    \centering
    \caption{Cumulative reward  of the first 500 frames among different attack methods on Torcs}
    \vspace{-10pt}
    \label{tab:result-torcs}
        \begin{Huge}\resizebox{\linewidth}{!}{%
    \begin{tabular}{c|c|ccc|cc|cc|cc|cc|cc|cc|cc}
         \toprule
        \multirow{1}{*}{\shortstack{non-adv}} &      \multirow{1}{*}{\shortstack{$\epsilon$}} & 
             \multirow{1}{*}{\shortstack{\textsf{ obs-fgsm-wb}}}  &
             \multirow{1}{*}{\shortstack{\textsf{ obs-fgsm-bb}}}  &
             \multirow{1}{*}{\textsf{\bf obs-fd-bb}}  &  \multicolumn{2}{c|}{\textsf{\bf obs-s[n]fd-bb} } & \multicolumn{10}{c|}{\textsf{online sequential attack} } & 
              \multicolumn{2}{c}{\textsf{\bf obs-seq[L60]-s[n]fd-bb}} \\
              
              & & & & & & &  \multicolumn{2}{c|}{\textsf{ obs-seq[$k$]-rand-bb}} &\multicolumn{2}{c|}{\textsf{   obs-seq[F$k$]-fgsm-wb}} &\multicolumn{2}{c|}{\textsf{ \bf obs-seq[F$k$]-fd-bb}} &\multicolumn{2}{c|}{\textsf{  obs-seq[S$k$]-fd-bb}} &\multicolumn{2}{c|}{\textsf{\bf  obs-seq[L$k$]-fd-bb}}  & \\
              \midrule
              \multirow{8}{*}{$571.4$} & \multirow{4}{*}{ 0.05}  &\multirow{4}{*}{ 5.8} & \multirow{4}{*}{45.2}& \multirow{4}{*}{ 11.9 }& n=10& 220.6& \multirow{2}{*}{k=10} & \multirow{2}{*}{581.74} &\multirow{2}{*}{ k=10} & \multirow{2}{*}{8.57} & \multirow{2}{*}{k=10} & \multirow{2}{*}{9.40} & \multirow{2}{*}{k=10} & \multirow{2}{*}{571.63} &\multirow{2}{*}{ k=10} & \multirow{2}{*}{400.82}  & n=10 & 492.2\\
              & & & & & \multirow{2}{*}{n=20}&\multirow{2}{*}{89.8} & & & & & & & & & & & n=20 & 483.0  \\
               & & & & &  & & \multirow{2}{*}{k=60}& \multirow{2}{*}{604.18} &\multirow{2}{*}{k=60}& \multirow{2}{*}{22.62} &\multirow{2}{*}{k=60}& \multirow{2}{*}{27.41}  &\multirow{2}{*}{k=60}& \multirow{2}{*}{302.52} &\multirow{2}{*}{{\bf k=60}}& \multirow{2}{*}{{\bf 30.27} } & n=40 & 334.4 \\
                & & & & &  {\bf n=40}& {\bf 43.2}  & & & & & & & & & & & {\bf n=100} &{\bf 28.7}\\
                  \cmidrule(l){2-19}

             & \multirow{4}{*}{0.10} &\multirow{4}{*}{5.9} & \multirow{4}{*}{18.6}& \multirow{4}{*}{14.9} &\multirow{1}{*}{n=10} &\multirow{1}{*}{22.3} & \multirow{2}{*}{k=10} &\multirow{2}{*}{583.59} & \multirow{2}{*}{ k=10}&\multirow{2}{*}{8.58} &\multirow{2}{*}{k=10}&\multirow{2}{*}{8.47} &\multirow{2}{*}{ k=10}&\multirow{2}{*}{525.53} &\multirow{2}{*}{k=10}& \multirow{2}{*}{386.61} & n=10 &433.2 \\
              &  & & &  &\multirow{2}{*}{n=20} & \multirow{2}{*}{23.6}  & & & & & & & & & & &n=20 & 391.2 \\
               &  & & &  & &  &\multirow{2}{*}{k=60} &\multirow{2}{*}{603.43} &\multirow{2}{*}{k=60} &\multirow{2}{*}{17.88} & \multirow{2}{*}{k=60} &\multirow{2}{*}{24.60} &\multirow{2}{*}{ k=60} &\multirow{2}{*}{271.73}& \multirow{2}{*}{{\bf k=60}} &\multirow{2}{*}{{\bf22.62}} & n=40 & 136.7  \\
               
                & & & & & {\bf n=40} & {\bf18.9}  & & & & & & & & & & & {\bf n=100 }& {\bf 28.2}\\
         \bottomrule
    \end{tabular}}
    \end{Huge}
    \vspace{-10pt}
\end{table*}

\section{Experiments}
\label{sec:evaluation}
We design our experiments to answer the following questions: \begin{enumerate*}
\item Can our proposed black-box method achieve similar or better performance compared with existing white-box and black-box methods?\item How does the adaptive SFD method perform compared with FD? \item How much improvement of sample efficiency does online sequential attack obtain? \item Does attacking the most important frame selection work better than attacking other frames? \item Does the RL based environment dynamics attack perform better than random search?
\item Does the real robot attack work in the visual navigation task?
\end{enumerate*}
To answer these questions, we first introduce the environments we use, and then introduce our baselines and evaluations on all methods.

\textbf{Experiment environments and victim RL models.}
We attack several agents trained for five different simulated RL environments:
Atari games including Pong and Enduro \cite{bellemare2013arcade},
HalfCheetah and Hopper in MuJoCo \cite{todorov2012mujoco}, and
the driving simulation TORCS \cite{pan2017virtual}.
We train DQN~\cite{mnih2015human} on Pong, Enduro and TORCS, and
 train DDPG~\cite{lillicrap2015continuous} on HalfCheetah and Hopper.
We report the cummulative reward on the first 500 frames. The reward function for TORCS comes from~\cite{pan2017virtual} and DQN network architecture comes from \cite{mnih2015human}. 
The network for continuous control
using DDPG comes from~\cite{dhariwal2017openai}.

\textbf{Baselines.}
 We compare the agents' performance under all attacks with their performance under no attack, denoted as \textsf{non-adv}.  We test the white-box attacks with FGSM~\cite{goodfellow2014explaining} (\textsf{obs-fgsm-wb}) and blackbox attacks with \textsf{obs-fgsm-bb}~\cite{huang2017adversarial} which the attacker leverages the transferability to perform attacks by  training a surrogate model to generate adversarial examples.
We test the attacks on observation under $L_\infty$ perturbation bounds of $\epsilon=0.005$ and $\epsilon=0.01$ on the Atari games and MuJoCo simulations and $\epsilon=0.05$ and $\epsilon=0.1$ on TORCS. (Observation values are normalized to [0,1].)

\textbf{Evaluating FD methods.}
We evaluate the finite difference method \textsf{obs-fd-bb} and test \textsf{obs-s[n]fd-bb} under different numbers $n$ of SFD iterations, for $n\in\{10,20\}$. Here we denote the attack that uses $n$ iterations as \textsf{obs-s[$n$]fd-bb}.
For online sequential attacks, we test
under conditions \textsf{obs-seq[F$k$]-fgsm-wb} and \textsf{obs-seq[F$k$]-fd-bb} (\textsf{F} for ``first''), where we use \textit{all} of the first $k$ frames to compute the gradient for generating a perturbation  and then apply the perturbation to the subsequent frames. We report the cumulative rewards of the first $k$ frames and the final cumulative reward among the first 500 frames by applying the perturbation after the first $k$ frames. We implement a baseline method \textsf{obs-seq[F$k$]-rand-bb} by using random noise as the perturbation to evaluate the effectiveness of our algorithms. To increase the attack efficiency,  we also evaluate the \textsf{obs-seq[L$k$]-fd-bb} (\textsf{L} for ``largest''), in which we select $20\%$ of the first $k$ frames that have the \textit{largest} $Q$ value variance to generate the universal perturbation, and  \textsf{obs-seq[S$k$]-fd-bb} (\textsf{S} for ``smallest''), in which we select 20\% of the first $k$ the frames that have the \textit{smallest} $Q$ value variance to generate the universal perturbation, as baseline. 

\textbf{Perturbations on components other than the observations}. For attacks on environment dynamics, we test \textsf{env-rand-bb} and \textsf{env-search-bb} on MuJoCo and TORCS. 
In the tests on MuJoCo, we perturb the body mass and body inertia vector, which are in $\mathbb{R}^{32}$ for HalfCheetah and $\mathbb{R}^{20}$ for Hopper. 
In the tests on TORCS, we perturb the road friction coefficient and bump size, which is in $\mathbb{R}^{10}$. The perturbation strength is within 10\% of the original magnitude of the dynamics being perturbed.

\textbf{Real robot experiment}.
We conduct physical attack experiments on an Anki Vector robot~\cite{anki}. For training the DRL policy, we design a discrete control task for the robot in a closed playground. The robot has a discrete action space of going forward, turning left, and turning right. It receives a positive reward of $3$ for moving forward (in any direction) and a reward of $-10$ for colliding with anything. The task ends if the robot collides with the wall. We train DQN policy until convergence. We set the maximum episode length to be 200 steps to shorten the training time. We use the attack method in Section~\ref{sec:attacks}'s real robot attack method to generate perturbation patches that are robust to the imperfections throughout the attack. For the target action $a'$ we choose the worst action under the original input $I$ for attack.
We print out the perturbed image $I'$ and crop a random patch from $I'$ and mount it in the robot's current field of view. In order to test the robustness of the attack algorithm, we mount the patch at different positions and put the robot at different viewing angles towards the patch.

\subsection{Experimental Results}

The first 6 columns of Table~\ref{tab:result-torcs} shows the results of the attacks on TORCS, including baseline attacks: \textsf{obs-fgsm-wb}, 
 and \textsf{obs-fgsm-bb}, and finite difference based methods (
\textsf{obs-fd-bb}, \textsf{obs-s[n]fd-bb}) on black-box settings. 
It shows that \textsf{obs-fd-bb} could achieve similar performance compared with whitebox attack (\textsf{obs-fgsm-wb}) and slightly better than the baseline blackbox attack ($\textsf{obs-fgsm-bb}$). Moreover, for \textsf{obs-s[$i$]fd-bb}, it shows the effectiveness and with $n$ increase, it will increase the computation cost but the attack effectiveness increases as well. For the results, we could observe that when $n=40$, it could achieve comparable attack effectiveness compared to $\textsf{obs-fd-bb}$, therefore, in the following experiments, we select  $n=40$ when we report $\textsf{obs-sfd-bb}$.  
In Table~\ref{tab:sfd-queries},  we show the number of queries for \textsf{obs-s[n]fd-bb} with different iteration parameter $n$ and the number of queries for \textsf{obs-fd-bb}.
The results show that \textsf{obs-sfd-bb} uses significantly fewer queries (around 1000 to 6000)
than \textsf{obs-fd-bb} (around 14,000) but achieves similar attack performance.

\begin{table}[h]
    \centering
    \vspace{-10pt}
    \caption{ Number of queries for SFD on each image among different settings. (14112 would be needed for FD.) }
    \label{tab:sfd-queries}
    \begin{tabular}{ccccc}
         \toprule
            $\epsilon$ & 10 iter. & 20 iter. & 40 iter. & 100 iter.  \\ \cmidrule(r){1-1} \cmidrule(l){2-5}
            
            0.05 & $1233\pm50$ & $2042 \pm 77$ & $3513\pm 107$ & $5926 \pm 715$ \\
            0.10 & $1234\pm41$ &$ 2028\pm 87$ & $3555 \pm 87$ & $6093 \pm 399$ \\
         \bottomrule
    \end{tabular}
    \vspace{-10pt}
\end{table}

Besides those, we  evaluate the performance of online sequential attacks (\textsf{obs-seq[F$k$]-fgsm-wb},
\textsf{obs-seq[F$k$]-fd-bb}, \textsf{obs-seq[L$k$]-fd-bb}, and baselines(
\textsf{obs-seq[S$k$]-fd-bb}, \textsf{obs-seq[F$k$]-rand-bb}) 
in the columns 8-17 of Table~\ref{tab:result-torcs} with different $L_{\infty}$ norm bound ($\epsilon=0.05$ and $0.1$).  

We could observe that the baseline \textsf{obs-seq[F$k$]-rand-bb} is not effective, while \textsf{obs-seq[F$k$]-fd-bb} achieves attack performance close to its white-box counterpart \textsf{obs-seq[F$k$]-fgsm-wb} and to non-online sequential attack \textsf{obs-fd-bb}. Even when $k=10$, the performance is still good. 
The 14-17 columns of Table~\ref{tab:result-torcs} shows the results of optimal frame selection.  We could observe that when we 
select a set of states with the largest Q value variance (\textsf{obs-seq[L$k$]-fd-bb}) to
estimate the gradient, the attack is more effective than selecting states with the  smallest Q value variance (\textsf{obs-seq[S$k$]-fd-bb}). It also empirically verifies corollary 1. When $k$ is very small ($k=10$), the estimated universal perturbation may be not strong enough to apply to the following frames while when $k=60$, the attack performance is reasonably good and similar to $\textsf{obs-seq[F60]-fd-bb}$. Therefore, in the following settings, we select $\textsf{obs-seq[L60]-fd-bb}$ as our default setting.

\begin{figure*}[t!]
    \centering
    \includegraphics[width=1.0\linewidth]{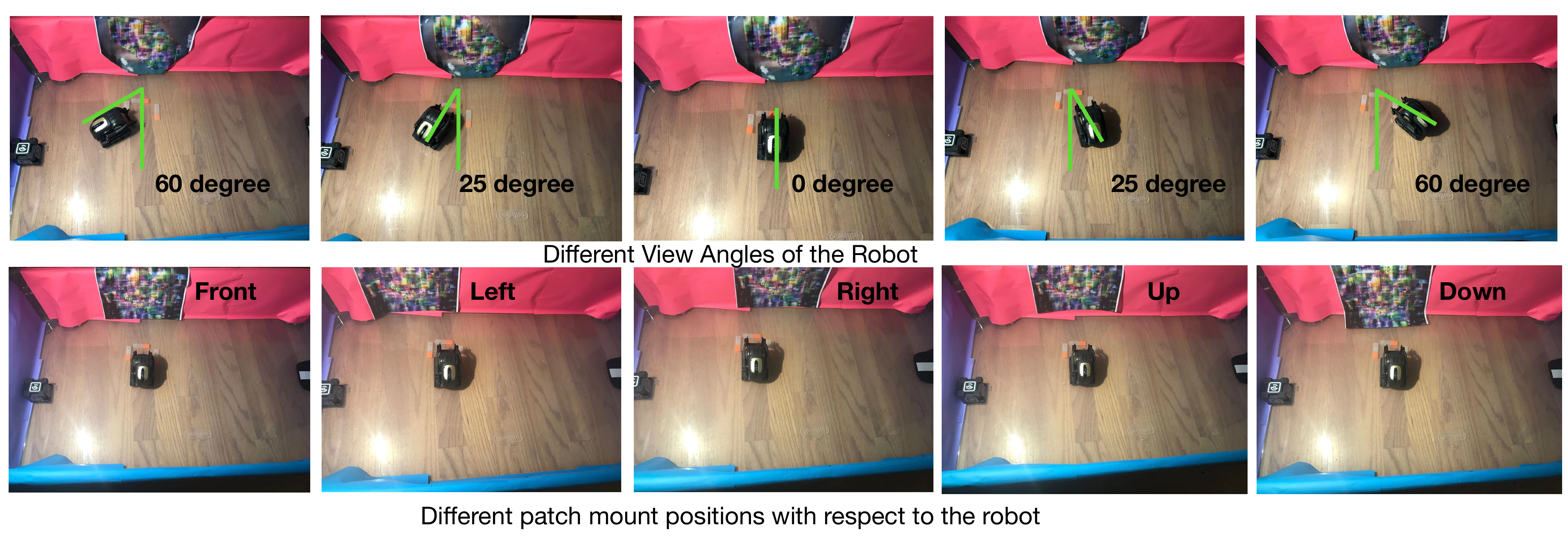}
    \vspace{-15pt}
    \caption{
    First row: example varied robot states with different viewing angles of the robot towards the patch. (The patch is generated using white-box method)
    Second row: example varied patching mounting positions with respect to the robot position. (The patch is generated using black-box method) We varied 50 different view angles from
    the left to the right to evaluate the robustness of the attack.
    }
    \label{fig:real_robot_location}
\end{figure*}

Finally, we combine online-sequential attack and sampling based finite difference together to evaluate the performance. We show the results of \textsf{obs-seq[L60]-s[$i$]fd-bb} by selecting the 20\% of frames with the largest Q value variance within the first 60 frames to estimate the gradient and using SFD with $i$ iterations.  From Table~\ref{tab:result-torcs}, we could observe that it is clear that with more iterations; we are able to get more accurate
estimation of the gradients and thus achieve better attack performance.
When $i=100$, it could achieve comparable attack effectiveness. By looking at Table~\ref{tab:sfd-queries},  we could find that  when  $i=100$, the number of queries for SFD is around 6k, which is still significantly smaller than needed for FD, which takes 14k
queries to estimate the gradient on an image of size $84\times 84$
(14112 = $84 \times 84 \times 2$). 

\begin{table}[h]
    \scriptsize
    \centering
    \vspace{-10pt}
    \caption{Results of different attacks on other environments. We report the cumulative reward within first 500 frames.}
    \label{tab:other-envs}
        \begin{small}\resizebox{\linewidth}{!}{%
    \begin{tabular}{c|c|ccc|cccc}
         \toprule
            \multicolumn{1}{c|}{\multirow{1}{*}{\shortstack{Env}}}  &  $\epsilon$ & \multirow{1}{*}{\shortstack{non-adv}} &
             \multirow{1}{*}{\shortstack{obs-fgsm-wb}}  &
                 \multirow{1}{*}{\shortstack{obs-fgsm-bb}}  &   
             \multirow{1}{*}{\bf obs-fd-bb}  &   \multirow{1}{*}{\bf obs-seq[L60]-s[100]fd-bb}  \\
            \midrule
            \multirow{2}{*}{Pong} &0.005 & \multirow{2}{*}{ 6.0} &-14.0 & -14.0&-13.0 &-9.0 \\
            & 0.01 &  & -14.0& -14.0 & -13.0 & -9.0\\
            \hline
            \multirow{2}{*}{Enduro} &0.005 & \multirow{2}{*}{ 43.0} &2.0 & 5.0&2.0 &27.0\\
            & 0.01 & & 2.0 & 5.0& 2.0 & 11.0 \\
            \hline 
            \multirow{2}{*}{HalfCheetah} &0.005 & \multirow{2}{*}{ 8257.1} &3447.6 &  2149.7&  6173.8 & 8263.6\\
            & 0.01 &  & 980.1&  1021.5& 1273.6 & 1498.4\\
            \hline
            \multirow{2}{*}{Hopper} &0.005 & \multirow{2}{*}{ 3061.4} &1703.0 & 1736.8& 1731.2 & 1843.5\\
            & 0.01 &  & 1687.2&1694.4 & 1711.3 &1718.8 \\
         \bottomrule
    \end{tabular}}
    \vspace{-10pt}
    \end{small}
\end{table}
We provide the results of attack applied on observation
space in other environments in Table~\ref{tab:other-envs}. These environments include Atari games Pong and Enduro, and MuJoCo robotics simulation environments HalfCheetah and Hopper. The results further instantiates the effectiveness of the set of proposed FD methods.

\paragraph{\textbf{Perturbations on components other than the observations}}
\textbf{Attacks on environment dynamics}.
In Table~\ref{tab:dyn_attack}, we show our results for 
performing targeted adversarial environment dynamics attack. The results are the $L_2$ distance to the target state (the smaller the better). 
The results show that random 
search method performs worse than RL based search method in terms of reaching
a specific state after certain steps. 
The quality of the attack can be
qualitatively evaluated by observing the sequence of 
states when the agent is being attacked and see whether
the target state has been achieved. The results are shown in figure ~\ref{fig:env-torcs}. We could observe that the final stages of our RL based method (\textsf{env-search-bb}) are similar to the targeted state among different games. We include example trajectories of agents under dynamics attack in our appendix.

\begin{figure}[h]
\begin{minipage}{0.50\textwidth}
\centering
\vspace{-5pt}
    \captionsetup{type=table}
\caption{Results of environment dynamics based attacks showing our proposed \textsf{env-search-bb} outperforms baseline \textsf{env-rand-bb}. Shown are the L2 distance to the target state, the smaller the better.}
\vspace{-5pt}
    \label{tab:dyn_attack}
\begin{small}
    \begin{tabular}{ccc}
         \toprule
         Environment & \textsf{env-rand-bb} & \textsf{env-search-bb} \\
         \cmidrule(r){1-1} \cmidrule(l){2-3}
         HalfCheetah & 7.91 & \textbf{5.76} \\
         Hopper & 1.89 & \textbf{0.0017}\\
         TORCS & 25.02 & \textbf{22.75}\\
         \bottomrule
    \end{tabular}
    \end{small}
    \vspace{-10pt}
\end{minipage}
\end{figure}

\begin{figure}[t!]
    \vspace{-2pt}
    \centering
    \begin{subfigure}{\linewidth}
    \centering
    \includegraphics[width=0.12\textwidth]{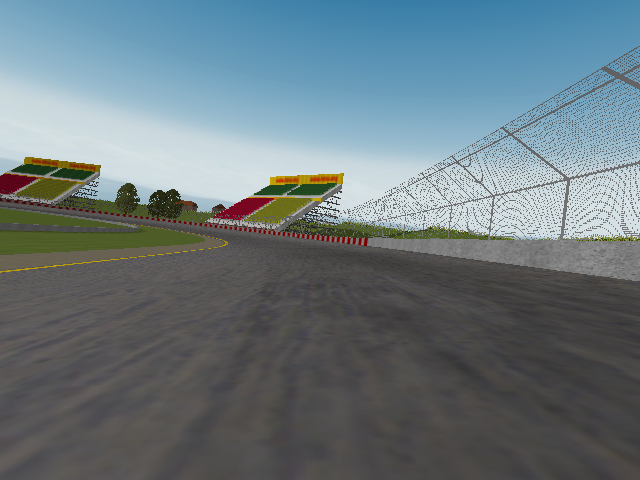}
    \includegraphics[width=0.12\textwidth]{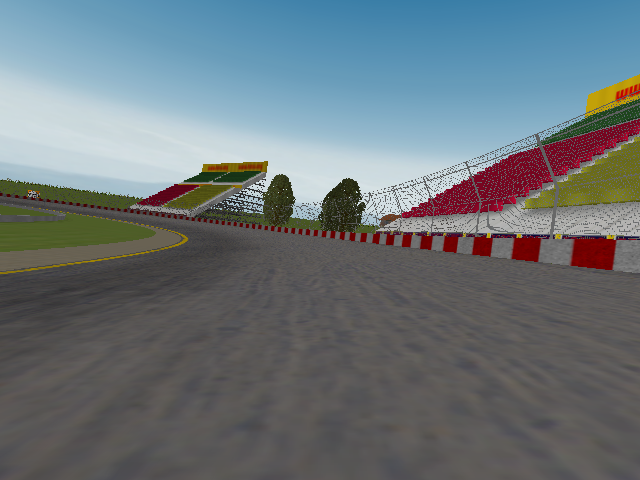}
    \includegraphics[width=0.12\textwidth]{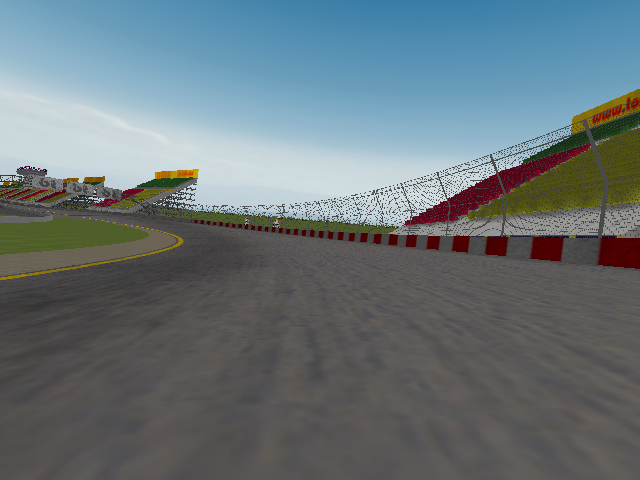}
    \includegraphics[width=0.12\textwidth]{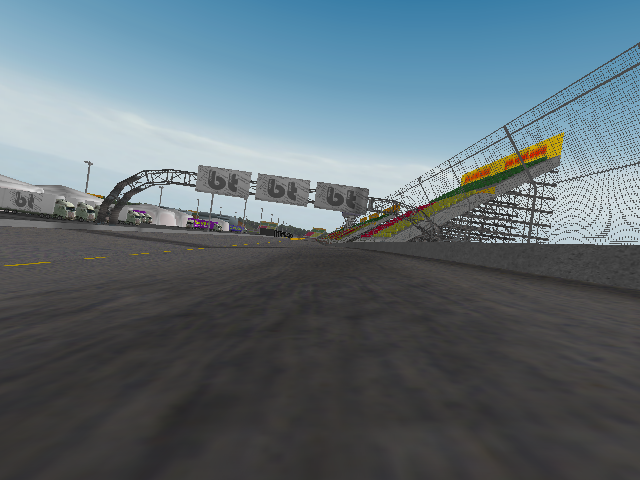}
    \includegraphics[width=0.12\textwidth]{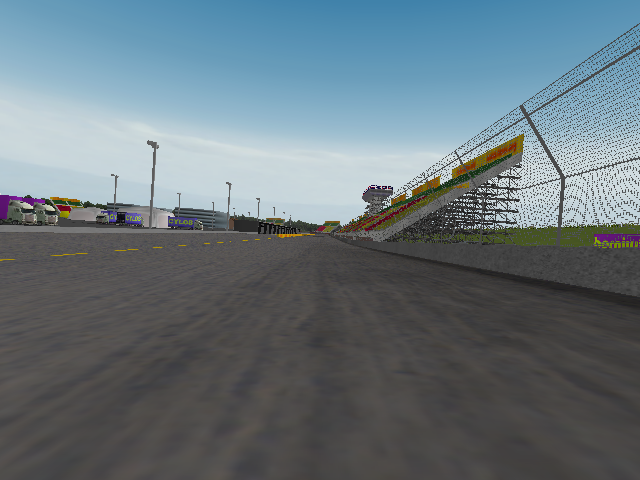}
    \includegraphics[width=0.12\textwidth]{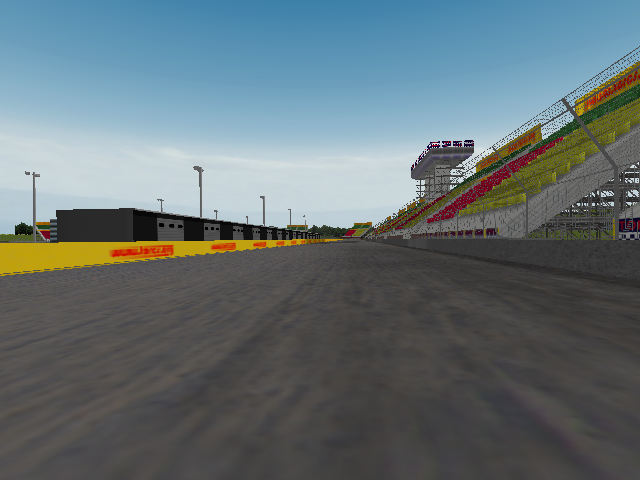}
    \includegraphics[width=0.12\textwidth]{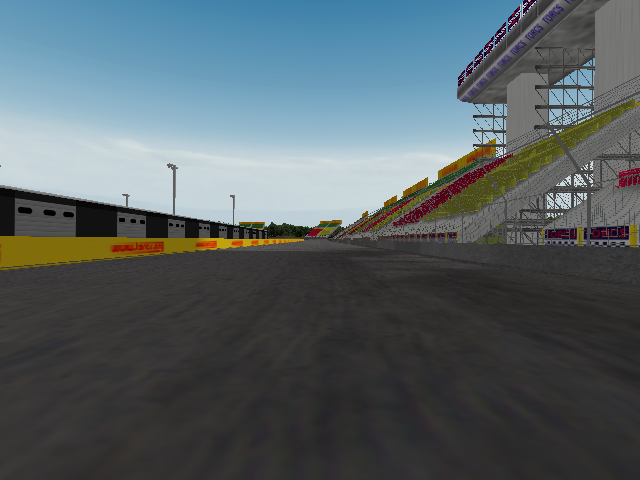}
    \caption{Agent's behavior under normal dynamics}
    \end{subfigure}
    \begin{subfigure}{\linewidth}
    \centering
    \includegraphics[width=0.12\textwidth]{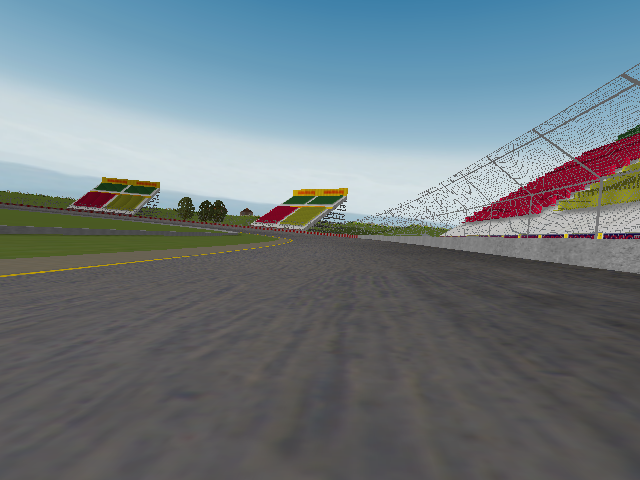}
    \includegraphics[width=0.12\textwidth]{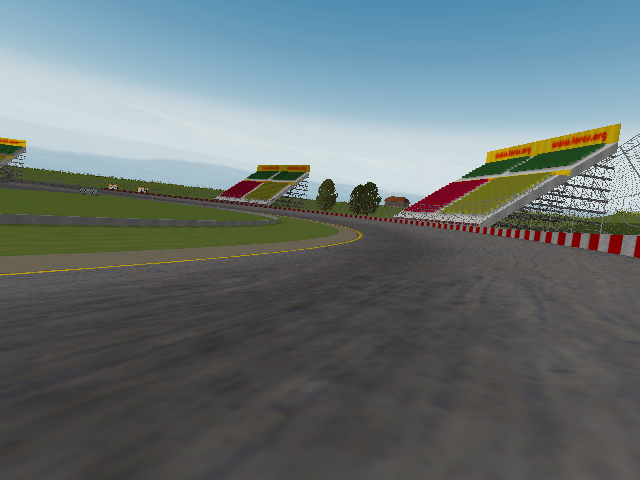}
    \includegraphics[width=0.12\textwidth]{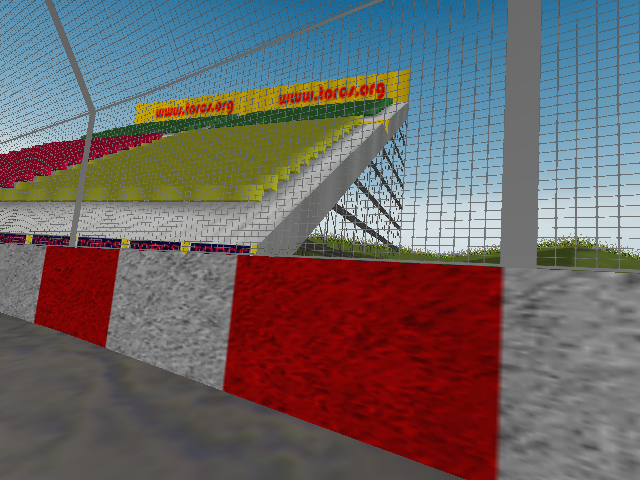}
    \includegraphics[width=0.12\textwidth]{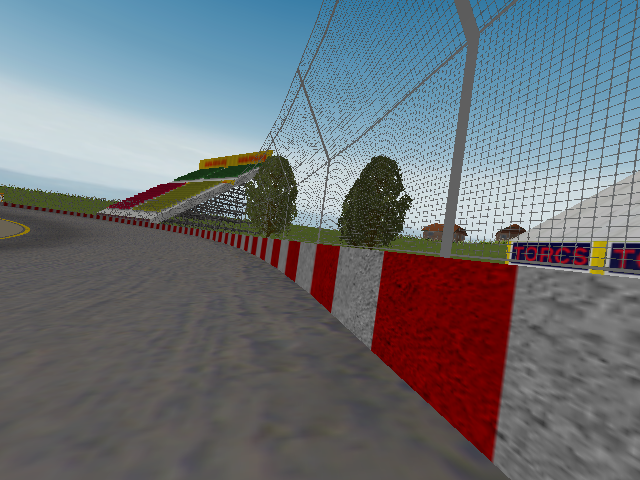}
    \includegraphics[width=0.12\textwidth]{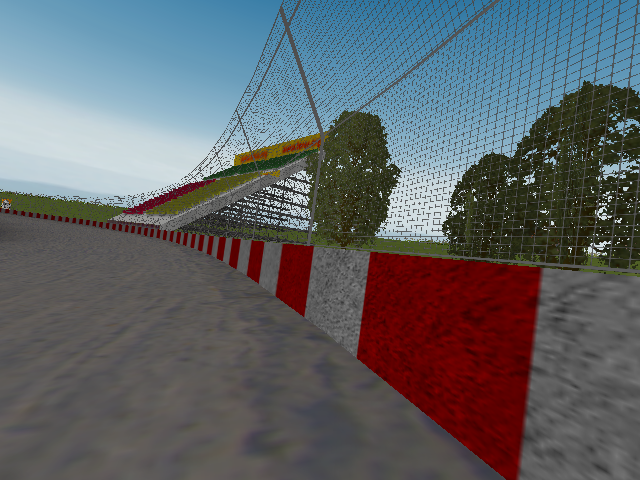}
    \includegraphics[width=0.12\textwidth]{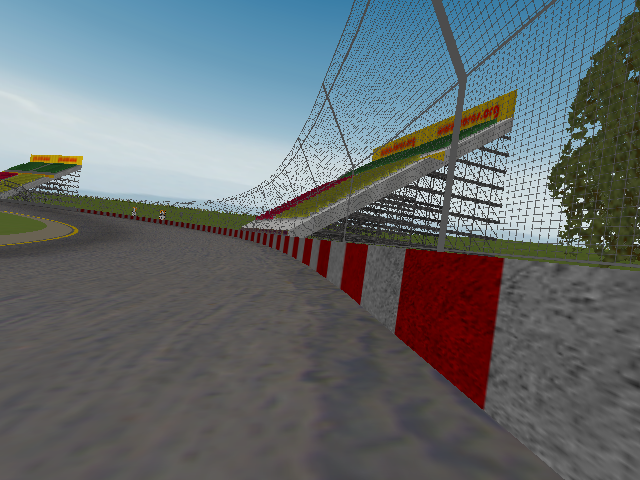}
    \includegraphics[width=0.12\textwidth]{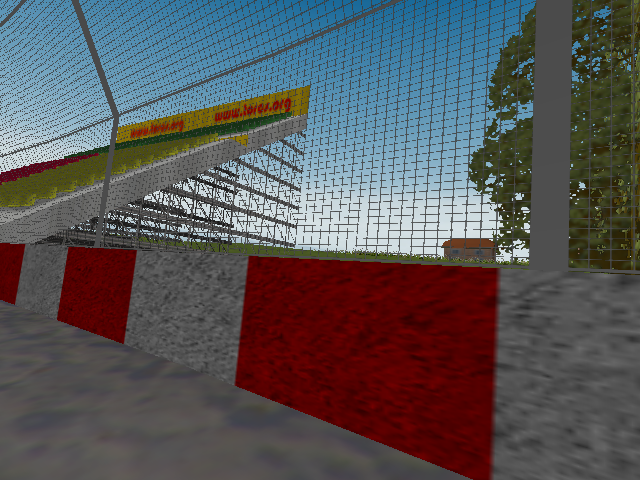}
    \caption{Agent's behavior under abnormal dynamics}
    \end{subfigure}
    \begin{subfigure}{\linewidth}
    \centering
    \includegraphics[width=0.12\textwidth]{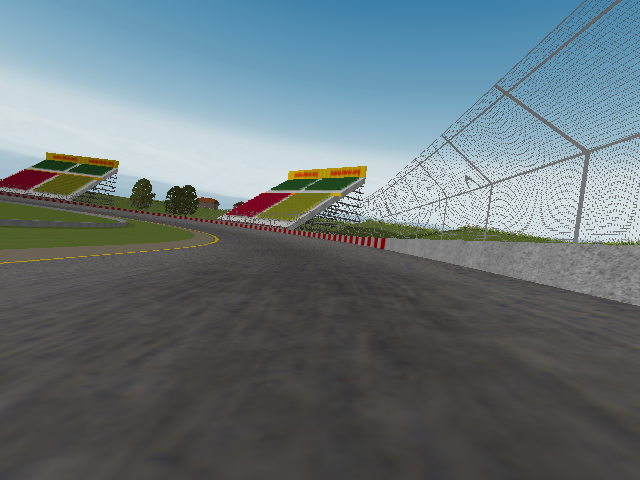}
    \includegraphics[width=0.12\textwidth]{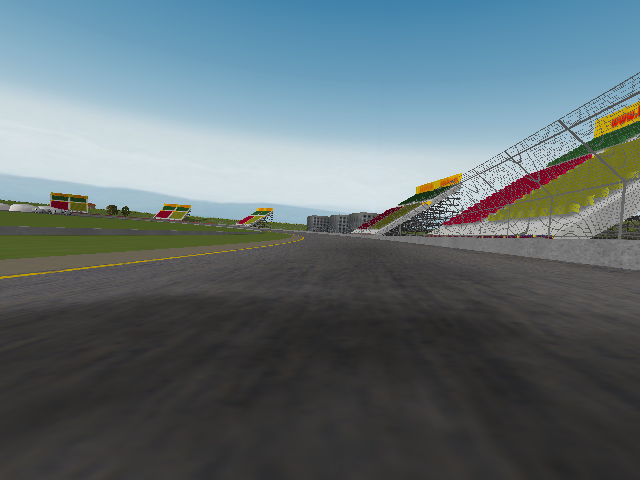}
    \includegraphics[width=0.12\textwidth]{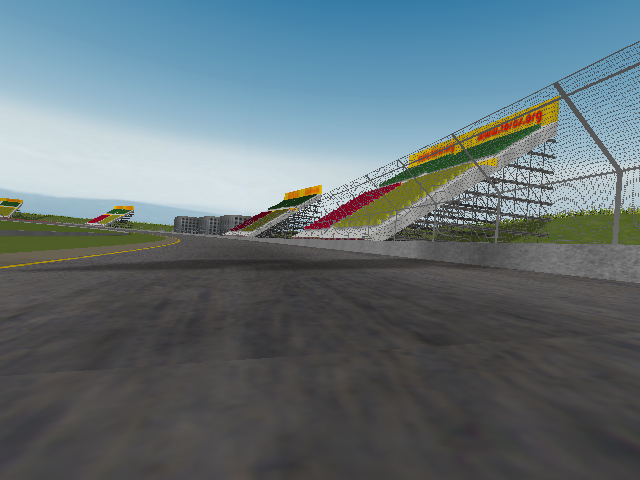}
    \includegraphics[width=0.12\textwidth]{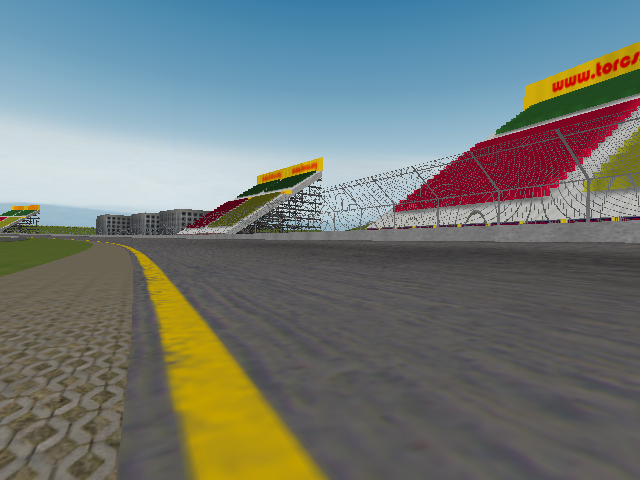}
    \includegraphics[width=0.12\textwidth]{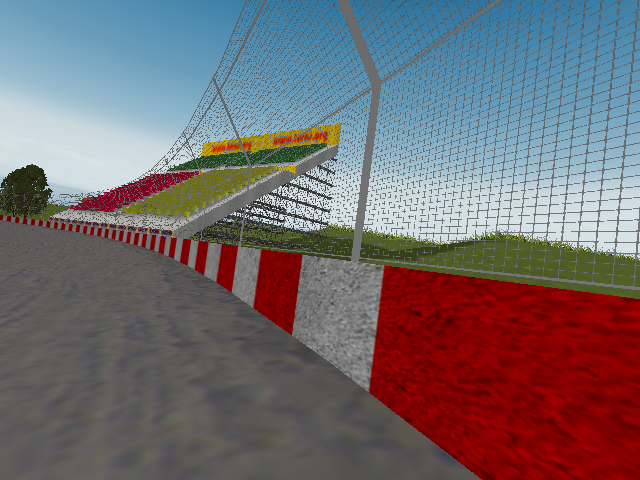}
    \includegraphics[width=0.12\textwidth]{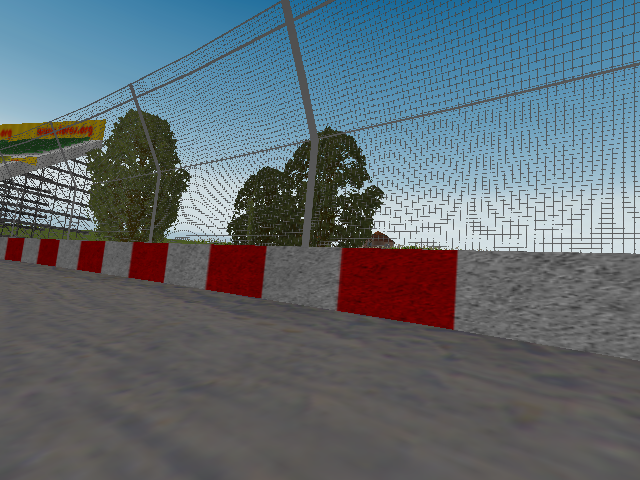}
    \includegraphics[width=0.12\textwidth]{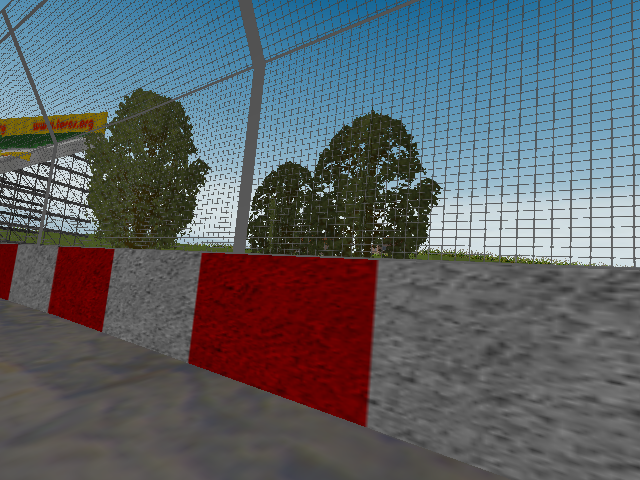}
    \caption{Agent's behavior under attacked dynamics using RL}
    \end{subfigure}
    \begin{subfigure}{\linewidth}
    \centering
    \includegraphics[width=0.12\textwidth]{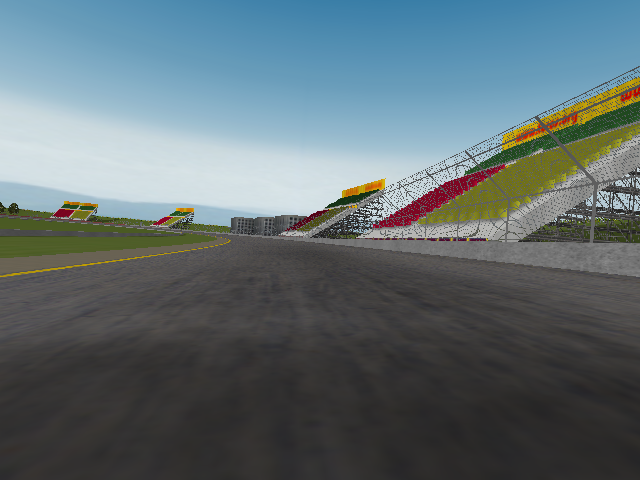}
    \includegraphics[width=0.12\textwidth]{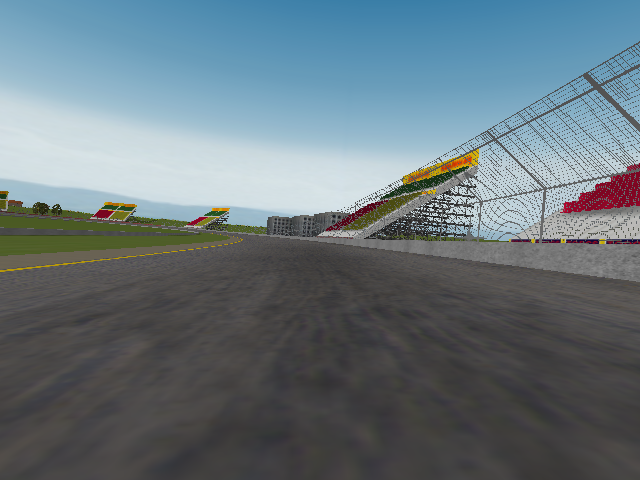}
    \includegraphics[width=0.12\textwidth]{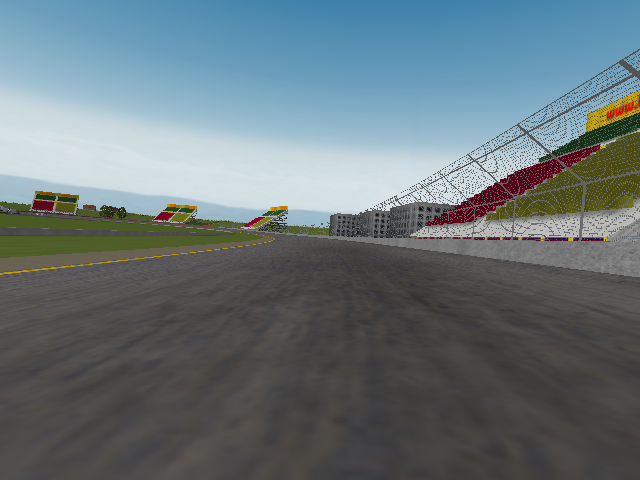}
    \includegraphics[width=0.12\textwidth]{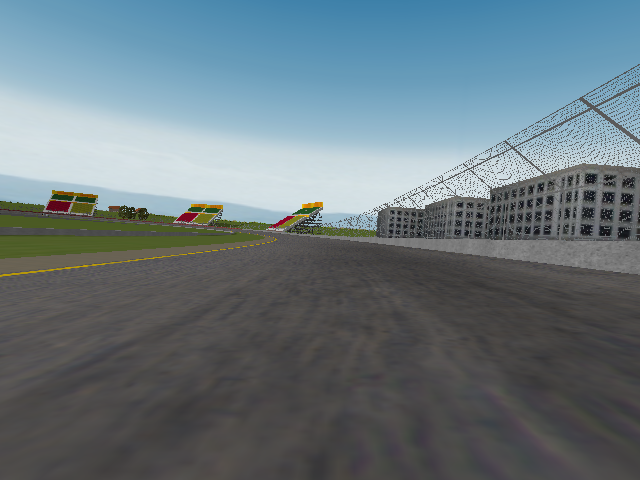}
    \includegraphics[width=0.12\textwidth]{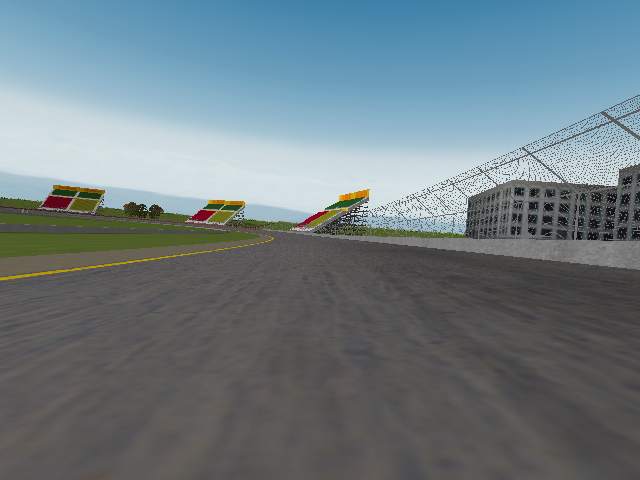}
    \includegraphics[width=0.12\textwidth]{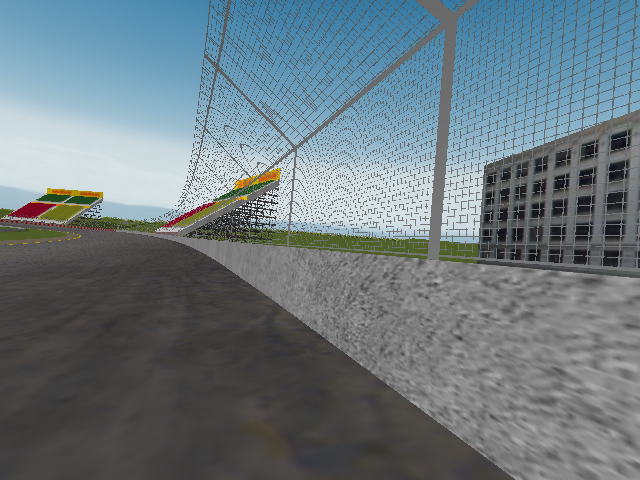}
    \includegraphics[width=0.12\textwidth]{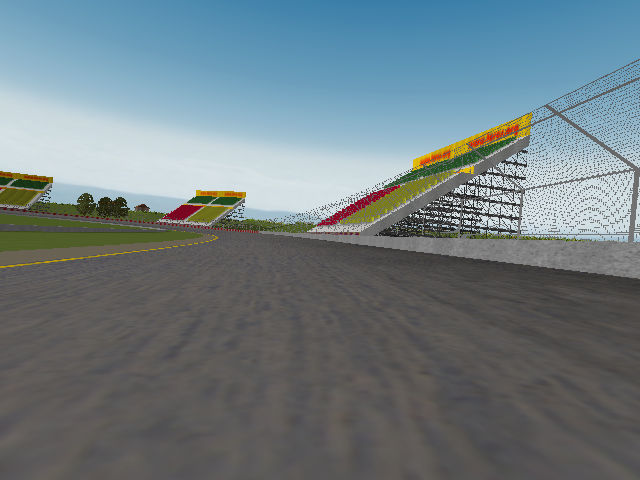}
    \caption{Agent's behavior under attacked dynamics using random search}
    \end{subfigure}
    \vspace{-0.4cm}
    \caption{Results for Dynamics Attack on TORCS}
    \label{fig:env-torcs}
\end{figure}

\textbf{Real robot policy attack}
To evaluate the performance of the adversarial patch in the physical world, we print it out and mount it at five different positions (front, left, right, up and below of the robot position) and put the robot at 50 different viewing angles from the adversarial patch (varied from -60$^{\circ}$ to 60$^{\circ}$). 
We select a target action for attack, which results in suboptimal behavior such as a collision. The physical experiment settings are included in Figure~\ref{fig:real_robot_location}.
The results of attack success rate are shown in Table~\ref{tab:real_robot}.
The results show that by mounting the patch just in front of the robot results in the best attack success rate both in white-box attack and black-box attack. When there are inaccuracies of patch mounting, such as mounting the patch to the slight left, right, up or below of the robot field of view, the attack success rate goes down.
\vspace{-5pt}

\begin{table}[h!]
    \centering
    \caption{Attack success rate over all viewing angles by mounting the adversarial image patch with different selected positions.}
    \vspace{-5pt}
    \begin{tabular}{l|ccccc}
    \toprule
    \multirow{2}{*}{Method } & \multicolumn{4}{c}{ Positions}\\
     & Front & Left & Right & Up & Below \\
    \midrule
    \textsf{obs-seq-fgsm-wb} & 55\% & 33\% & 37\% & 42\% & 48\% \\
    \textsf{obs-seq-sfd-bb} & 50\% & 29\% & 31\% & 33\% & 38\%  \\
    \bottomrule
    \end{tabular}
    \label{tab:real_robot}
\end{table}

\section{Discussion and Conclusions}
\label{sec:conclusion}
We propose and evaluate a set of black-box adversarial attacks on DRL models using finite-difference methods. The sample efficiency of FD based methods are further improved by using the adaptive sampling method, online sequential attack and attack frame selection. Most importantly, we provide the first example of adversarial attacks on real world visual navigation robot. Studying the adversarial attack on real robot can help to find potential defense approaches to improve the robustness of DRL models. From the perspective of training robust RL policies, it is important to know the severeness of the risk related with the proposed attacks. Among the proposed attacks, the environment dynamics attack can be a more realistic potential risk to consider than the attacks based on observations. The reason is that this attack does not require access to modify the policy network software system and only requires access to modify environment dynamics, and by modifying the environment dynamics parameters such as changing road condition in autonomous driving, we see from our experiments that the agent tends to fail with the original policy. The observation space based attack, especially the black-box attacks, are also important to defend against, since an attacker can definitely query the network and may have access to change the observations. The real world physical attack experiment by modifying the environment poses a real concern for DRL policies that may be potentially affected. 
We hope our exploratory work and the analysis of attacks help form a more complete view for what threats should be considered in ongoing research in robust reinforcement learning.

\begin{acks}
\vspace{-5pt}
 This work is partially supported by the NSF grant No.1910100, NSF CNS 20-46726 CAR, Berkeley Deep
Drive (BDD), Berkeley Artificial Intelligence Research, Open Philanthropy and Amazon Research Award. We thank for discussions from the Biomimetic Millisystems Group at UC Berkeley.
\end{acks}

\bibliographystyle{ACM-Reference-Format}
\balance
\bibliography{adrl}
\clearpage

\appendix

\section{Adaptive Sampling-based FD Algorithm}
We give the SFD algorithm description in the following Algorithm~\ref{alg:sfd}.

\begin{algorithm}[h!]
  \caption{Adaptive sampling based finite difference (SFD)}
  \label{alg:sfd}
  \begin{algorithmic}
  \STATE{\bfseries Input:} \\
  \quad\quad\quad $\textbf{s}\in\mathbb{R}^d$: state vector\;  \\
  \quad\quad\quad $f(w)$: loss function with parameters $w$\; \\
  \quad\quad\quad $k$: \# of item to estimate gradient at each step \; \\
 \quad\quad\quad $n$: \# of iteration\; \\
 \quad\quad\quad $\theta$: the gradient threshold \\
 \quad\quad\quad $\delta$: finite difference perturbation value \; \\
    {\bfseries Output:} estimated gradient $\nabla_{\textbf{s}}\hat{f}(\textbf{s};w)$ \\
    {\bfseries Initialization:}\\
    \quad $\nabla_{\textbf{s}}\hat{f}(\textbf{s};w)\leftarrow 0$, randomly select $k$ dimensions in $\{1,2,\cdots, d\}$ to form an index set $P$. \\
  \textbf{For} {$t=0$ {\bfseries to} $n$}\\
  \quad\quad \textbf{For} {$j\in P$}\\
  \quad\quad\quad\quad \textbf{If} {$\nabla_j\hat{f}(\textbf{s};w)$hasn't been estimated}\\
  \quad\quad\quad\quad\quad\quad Get $\textbf{v}\in\mathbb{R}^d$ such that $\textbf{v}_j=1$ and $\textbf{v}_i=0, \forall i\neq j$ \\
  \quad\quad\quad\quad\quad\quad obtain $\nabla_j\hat{f}(\textbf{s};w) = \frac{f(\textbf{s}+\delta\textbf{v};w)-f(\textbf{s}-\delta\textbf{v};w)}{2\delta}$ \\
  \quad\quad\quad\quad \textbf{end}\\
  \quad\quad \textbf{end}\\
  \quad\quad $P' = \{k; k\in P,|\nabla_k\hat{f}(\textbf{s};w)|\geq\theta\}$  \\
  \quad\quad $P$ = indexes of neighbors of indexes in $P'$ \\
  end
  \end{algorithmic}
\end{algorithm}

\section{Dynamics Attack Algorithm}

\begin{algorithm}[h!]
  \caption{Dynamics Attack Algorithm}
  \label{alg:dynamics_attack}
  \begin{algorithmic}
  \STATE{\bfseries Input:} \\
  \quad\quad\quad $\pi$: well trained policy\; \\
  \quad\quad\quad $env$: the environment used to train the original policy\; \\
  \quad\quad\quad $T$: total training steps\;\\
  \quad\quad\quad $s_{target}$: target attack state\; \\
  \quad\quad\quad $n$: maximum episodic steps \\
  {\bfseries Initialization:}\\
    \quad $\pi_{attack}$: RL policies for generating dynamics perturbations \\
    \quad $B$: replay buffer \\
    \quad $d_{current}$ = $env.$dynamics\\
  \textbf{For} {$t=0$ {\bfseries to} $T$} \\
    \quad\quad $d_{new}$ = $\pi_{attack}$.sample\_action() \\
    \quad\quad $env.$change\_dynamics($d_{new}$) \\
    \quad\quad reached\_state = evaluate\_policy($env, \pi$) \\
    \quad\quad reward = $1/\|$reached\_state-$s_{target}\|$ \\
    \quad\quad done = episode steps reach $n$ \\
    \quad\quad action = $d_{new}$ \\
    \quad\quad B.append($d_{current}$, $d_{new}$, action, reward, done) \\
    \quad\quad train\_policy($\pi_{attack}$, B) \\
    \quad\quad $d_{current}$ = $d_{new}$ \\
  \end{algorithmic}
  \label{algorl}
\end{algorithm}

\textbf{Details of dynamics attack on Torcs environment.} In  the Torcs self-driving car environment, we modified the road friction parameter and road bump size to generate novel challenging environments. Specifically, parameters within the XML files that defines these road properties will be modified and recompiled each time a new environment is generated.

\section{Proof of Lemma 1}
\begin{proof}
Following the notation in Lemma~\ref{co2}, define
\begin{equation}
\centering
\begin{split}
    P_\textit{random}(\beta+\theta) & = \frac{1}{|S_A|}\sum_{i \in S_A}1(|\nabla_if(x)|\geq\beta+\theta) \\
    P_\textit{random}(\theta) & = \frac{1}{|S_A|}\sum_{i \in S_A}1(|\nabla_if(x)|\geq\theta) \\
    \end{split}
\end{equation}
Note that the randomness is with respect to the dimensions.
These probabilities are, when we randomly choose
$k$ dimensions in a gradient vector of $d$ dimensions,
the chance of sampling some dimension with
gradient absolute value no less than $\beta+\theta$
or $\theta$, respectively.
For $\nabla_if(x)$, its neighbor gradient's absolute value follows the following distribution: 

Define $A_0 = P_\textit{random}(\theta+\beta)k, B_0=[P_\textit{random}(\theta)-P_\textit{random}(\theta+\beta)]k$ as the number 
of nontrivial gradients estimated in iteration 0, and define
\begin{equation}
    \centering
    \begin{split}
        A_t & = (1-q) * A_{t-1} + p_3*B_{t-1} \\
        B_t & = q * A_{t-1} + p_2 * B_{t-1},
    \end{split}
\end{equation}
as the number 
of nontrivial gradients estimated in iteration $t$,
then the ratio 
$$R_t = \frac{A_t+B_t}{A_{t-1}+B_{t-1}}$$
characterizes in every iteration, the ratio of nontrivial gradient estimation over the total number of gradient estimation for $t\geq 1$, while $P_\textit{random}(\theta)$ characterizes nontrivial gradient estimation ratio if using random sampling method. For $t=0$, $R_0=P_\textit{random}(\theta)$. Since $R_t$ is the
ratio of the number non-trivial gradients over
the total number of gradients estimated
in the $t$-th iteration, if we randomly select
the dimensions, $R_t$ will be the same as
$P_\textit{random}(\theta)$. Using our SFD method,
if for every $t>0$, we have $R_t\geq P_\textit{random}(\theta)$,
then overall speaking, the ratio of the number
of non-trivial gradients estimated over the
total number of gradients estimated, $P_\textit{SFD}(\theta)$,
will be greater than $P_\textit{random}(\theta)$. 
The definition of $P_\textit{SFD}(\theta)$ is
\begin{equation}
    P_\textit{SFD}(\theta) = \frac{\sum_{i=1}^nA_t+B_t}{\sum_{i=0}^{n-1}A_t+B_t},
\end{equation}
where $n$ is the number of iterations. 

We now prove that if $p_1\leq 1-P_\textit{random}(\theta)$, then 
for $t>0$,
\begin{equation}
    R_t = \frac{A_t+B_t}{A_{t-1}+B_{t-1}} \geq P_\textit{random}(\theta).
\end{equation}
More specifically, for $t=0$, since we 
perform uniform random sampling, and we
sample $k$ dimensions, we have
\begin{equation}
    R_0 = \frac{P_\textit{random}(\theta)k}{k} = P_\textit{random}(\theta).
\end{equation}
Then if we can prove that for $t\geq 1$,
$R_t \geq P_\textit{random}(\theta)$, and for some $t$, we have $R_t > P_\textit{random}(\theta)$, it in turn proves $P_\textit{SFD}(\theta) > P_\textit{random}(\theta)$. The reason is 
that if for every step in our SFD algorithm the sample efficiency is at least the same as random sampling, and
for some iterations our SFD is more 
efficient, then overall speaking our SFD based gradient estimation is more efficient than random sampling based
gradient estimation.

Now from the above definition, we have
\begin{equation}
    \centering
    \begin{split}
      R_t & = \frac{A_t+B_t}{A_{t-1}+B_{t-1}} \\
      & = \frac{A_{t-1}(1-q)+p_3B_{t-1}+qA_{t-1}+p_2B_{t-1}}{A_{t-1}+B_{t-1}} \\
      & = \frac{A_{t-1}+(1-p_1)B_{t-1}}{A_{t-1}+B_{t-1}} \\
      & = 1 - \frac{p_1B_{t-1}}{A_{t-1}+B_{t-1}} \\
      & > 1 - p_{1}
    \end{split}
\end{equation}
Therefore, when $p_{1} < 1 - P_\textit{random}(\theta)$, $R_t > P_\textit{random}(\theta)$, our sampling algorithm is more efficient than random sampling. 
\end{proof}

\section{Proof of Theorem 2}
\begin{proof}
when $x\in\mathbb{R}^1$, assume function $f$ is $C^{\infty}$,
by Taylor's series we have
\begin{equation}
\begin{split}
    f(x+\delta) & = f(x) + f'(x)\delta + \frac{{\delta}^2}{2}f''(x)+\frac{h^3}{3!}f^{(3)}(x)+\cdots \\
    f(x-\delta) & = f(x) - f'(x)\delta + \frac{{\delta}^2}{2}f''(x)-\frac{{\delta}^3}{3!}f^{(3)}(x)+\cdots. \\
\end{split}
\end{equation}
Combine the two equations we get
\begin{equation}
    \frac{f(x+\delta)-f(x-\delta)}{2\delta} - f'(x) = 
    \sum_{i=1}^{\infty}\frac{\delta^{2i}}{(2i+1)!}f^{(2i+1)}(x),
\end{equation}
which means the truncation error is bounded by $O(\delta^2)$.
Moreover, we have
\begin{equation}
    \big|\frac{f(x+\delta)-f(x-\delta)}{2h}-f'(x)\big|\leq C\delta^2,
\end{equation}
where $C = \sup_{t\in[x-\delta_0,x+\delta_0]}\frac{f^{(3)}(t)}{6}$,
and $0<\delta\leq \delta_0$.

We can regard each dimension as a single variable function
$f(x_i)$, then we have
\begin{equation}
    \big|\frac{f(x_i+\delta)-f(x_i-\delta)}{2\delta}-f'(x_i)\big|\leq C\delta^2.
\end{equation}
Then $\forall i\in\{1,2,\cdots,d\}$, assume we are able to sample all nontrivial gradients with absolute gradient value no less than $\theta$, then we have
\begin{equation}
\begin{split}
    \sum_{i\in\{1,2,\cdots,d\}   ,\nabla_if(x)\geq\theta}\|\nabla_i\hat{f}(x)-\nabla_if(x)\|_1&\leq S_{\theta}C\delta^2 \\
    \sum_{i\in\{1,2,\cdots,d\}   ,\nabla_if(x)<\theta}\|\nabla_i\hat{f}(x)-\nabla_if(x)\|_1&\leq (d-S_{\theta})\theta.
    \end{split}
\end{equation} 

Therefore, the truncation error of
 gradients estimation is upper bounded by the following inequality.
 \begin{equation}
     \|\nabla\hat{f}(x)-\nabla f(x)\|_{1} \leq S_{\theta}C\delta^2+(d-S_{\theta})\theta,
 \end{equation}
for some $C>0$, and $\nabla\hat{f}(x)$ is the estimated gradient of $f$ with respect
to $x$.
\end{proof}

\section{Proof of Corollary 3}
\begin{proof}
Recall the definition of Q value is
\begin{equation}
    Q(s_{\tau},a_{\tau}) = \mathbb{E}_{\pi}[\sum_{t=\tau}^{H-1}\gamma^{t-\tau}r_t|s_{\tau},a_{\tau}].
\end{equation}
The variance of Q value at a state $s$ is defined as
\begin{equation}
    Var(Q(s)) = \frac{1}{|\mathcal{A}|-1}\sum_{i=1}^{|\mathcal{A}|}\big(Q(s,a_i)-\frac{1}{|\mathcal{A}|}\sum_{j=1}^{|\mathcal{A}|}Q(s,a_j)\big)^2,
\end{equation}
where $\mathcal{A}$ is the action space of the MDP, and $|\mathcal{A}|$ denotes the 
number of actions. 
Suppose we are to attack
state $s_m$ and state $s_n$ where the Q value variance of this
two states are $Var(Q(s_m),\cdot)$ and $Var(Q(s_n),\cdot)$, and assume 
$m<n$. 

Denote the state-action pair Q
values after attack are $Q(s_m,\hat{a}_m)$ and $Q(s_n,\hat{a}_n)$, respectively. During the attack, state $s_m$ is modified to $\hat{s}_m$, and state $s_n$ is modified to $\hat{s}_n$, and their action's Q-value also change, so we use $\hat{a}_m$ and $\hat{a}_n$ to denote the actions after the attack. By using $s_m$ and $s_n$ instead of $\hat{s}_m$ and $\hat{s}_n$, we mean that though the observed states are modified by the attack algorithm, but the true states do not change. By using a different action notation, we mean that since the observed states have been modified, the optimal actions at the modified states can be different from the optimal actions at the original observed states.
Then the total discounted expected
return for the entire episode can be expressed as (assume all actions are optimal actions)
\begin{equation}\label{cor21}
\begin{split}
    Q' & = Q(s_0,a_0)-\gamma^mQ(s_m,a_m)+\gamma^mQ(s_m,\hat{a}_m),\\
    Q'' & = Q(s_0, a_0) - \gamma^nQ(s_n, a_n) + \gamma^nQ(s_n, \hat{a}_n). \\
\end{split}
\end{equation}
Since $m<n$, $Q''$ can also be expressed as
\begin{equation}\label{cor22}
\begin{aligned}
    Q'' = &\ Q(s_0, a_0) - \gamma^mQ(s_m, a_m) + \gamma^mQ(s_m, a_m) \\
    & -
    \gamma^nQ(s_n, a_n) + \gamma^nQ(s_n, \hat{a}_n). 
\end{aligned}
\end{equation}
Subtract $Q'$ by $Q''$ we get
\begin{equation}
\begin{aligned}
    Q' - Q''  = &\ \gamma^m(Q(s_m,\hat{a}_m) - Q(s_m,a_m)) \\
        & +\gamma^nQ(s_n,a_n) - \gamma^nQ(s_n,\hat{a}_n) \\
     = &-\gamma^m[Q(s_m,a_m)-Q(s_m,\hat{a}_m) \\
     & - \gamma^{n-m}(Q(s_n,a_n)-Q(s_n,\hat{a}_n))].
\end{aligned}
\end{equation}
According to our claim that states where the variance of Q value
function is small will get better attack effect, suppose 
$Var(Q(s_m))>Var(Q(s_n))$, and assume the range of Q value at
step $m$ is larger than step $n$, then we have
\begin{equation}
\begin{aligned}
    Q(s_m, a_m) - Q(s_m, \hat{a}_m) &> Q(s_n, a_n) - Q(s_n, \hat{a}_n) \\&> \gamma^{n-m}[Q(s_n, a_n) - Q(s_n, \hat{a}_n)].
\end{aligned}
\end{equation}
Therefore $Q'-Q''<0$ which means attack state $m$ the agent will
get less return in expectation. If $Var(Q(s_m))<Var(Q(s_n))$, 
assume the range of Q value at step $m$ is smaller than step $n$, 
then we have
\begin{equation}
    Q(s_m, a_m) - Q(s_m, \hat{a}_m) < Q(s_n, a_n) - Q(s_n, \hat{a}_n).
\end{equation}
If $n-m$ is very small or $Q(s_n, a_n)-Q(s_n,\hat{a}_n)$ is large
enough such that $Q(s_m, a_m) - Q(s_m, \hat{a}_m) < \gamma^{n-m}[Q(s_n, a_n) - Q(s_n, \hat{a}_n)]$, then we have $Q'-Q''>0$ which means attacking state $m$ the agent will get more reward in expectation than attacking
state $n$.
\end{proof}

\begin{figure}[h!]
    \centering
    \begin{subfigure}{\linewidth}
    \centering
    \includegraphics[width=0.13\textwidth]{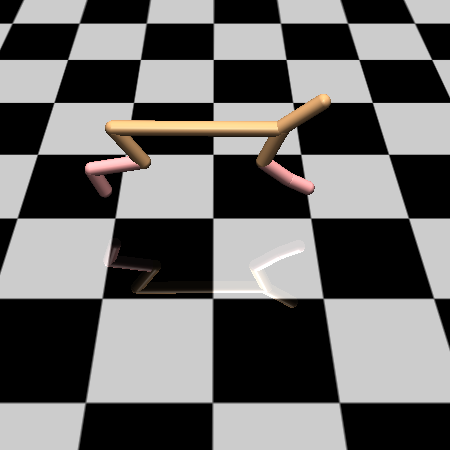}
    \includegraphics[width=0.13\textwidth]{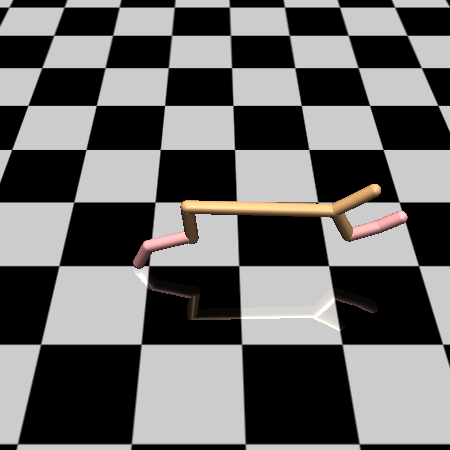}
    \includegraphics[width=0.13\textwidth]{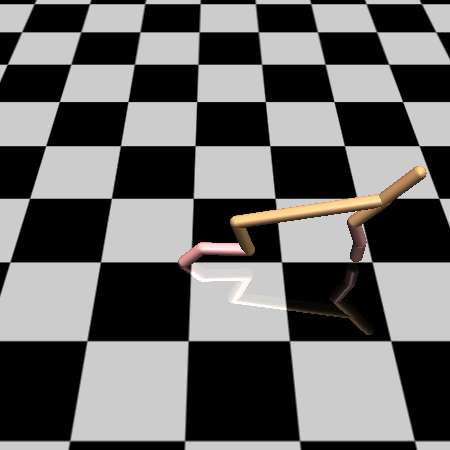}
    \includegraphics[width=0.13\textwidth]{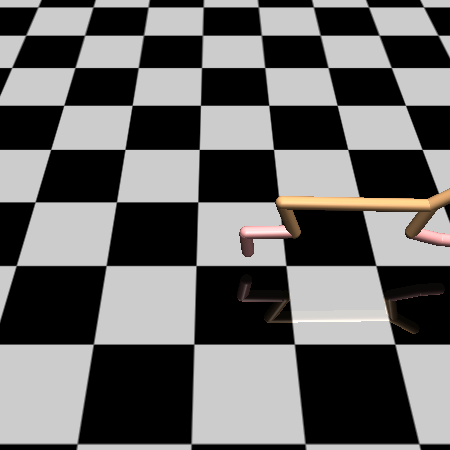}
    \includegraphics[width=0.13\textwidth]{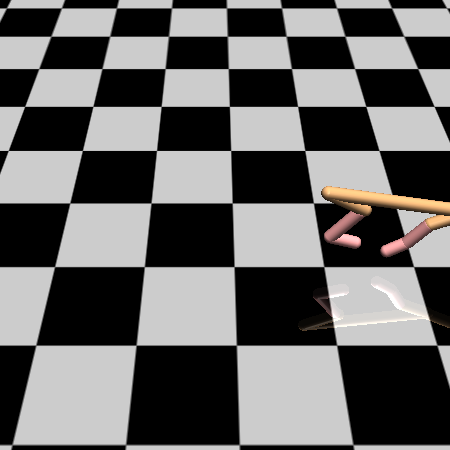}
    \includegraphics[width=0.13\textwidth]{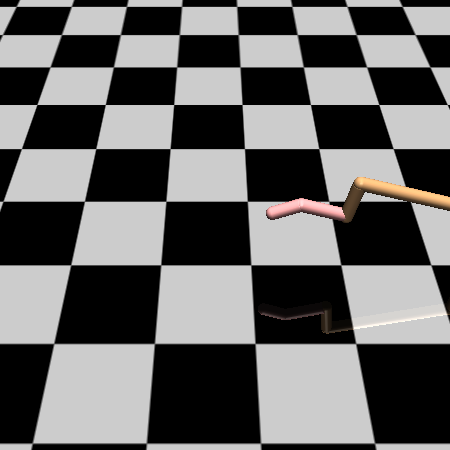}
    \caption{Agent's behavior under normal dynamics}
    \end{subfigure}
    \begin{subfigure}{\linewidth}
    \centering
    \includegraphics[width=0.13\textwidth]{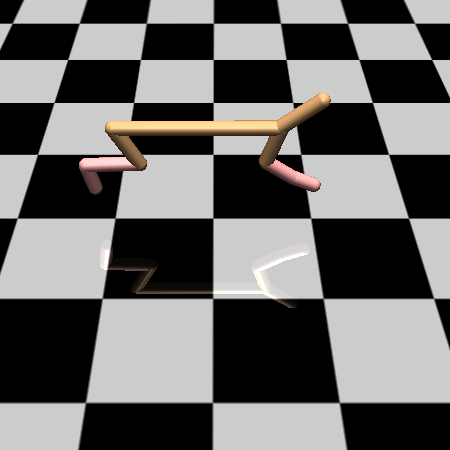}
    \includegraphics[width=0.13\textwidth]{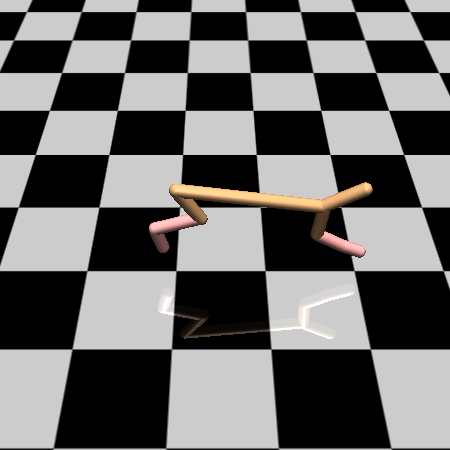}
    \includegraphics[width=0.13\textwidth]{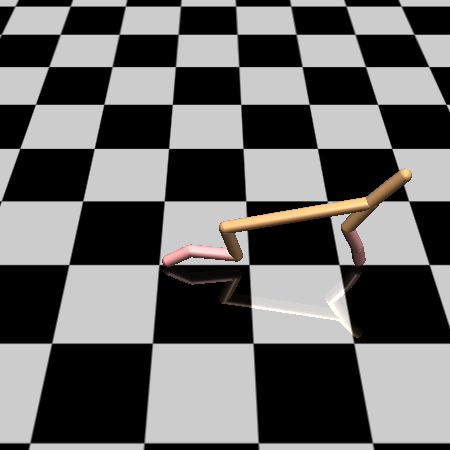}
    \includegraphics[width=0.13\textwidth]{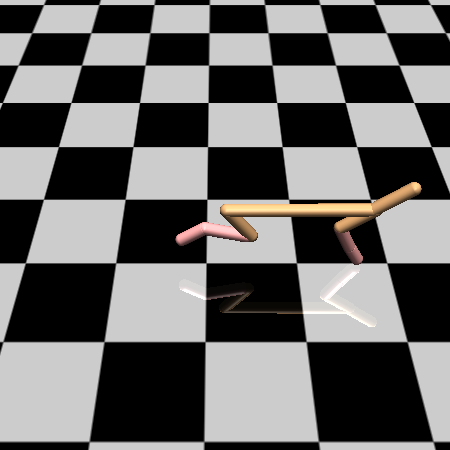}
    \includegraphics[width=0.13\textwidth]{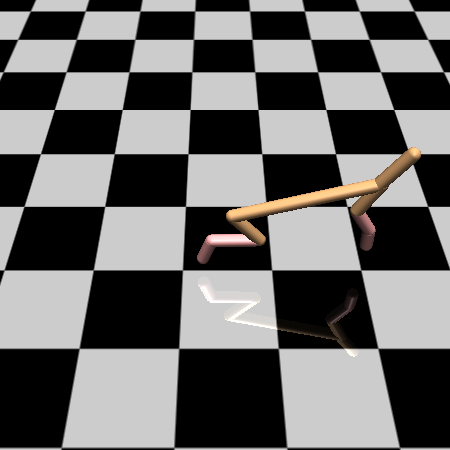}
    \includegraphics[width=0.13\textwidth]{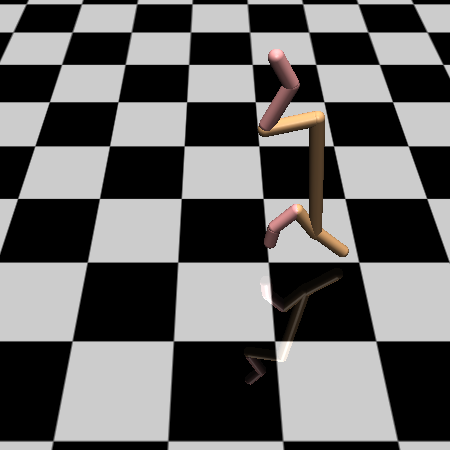}
    \caption{Agent's behavior under abnormal dynamics}
    \end{subfigure}
    \begin{subfigure}{\linewidth}
    \centering
    \includegraphics[width=0.13\textwidth]{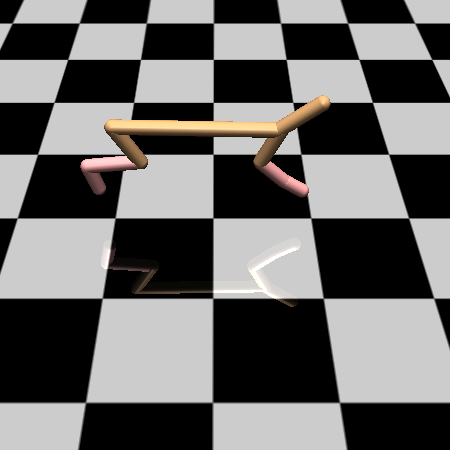}
    \includegraphics[width=0.13\textwidth]{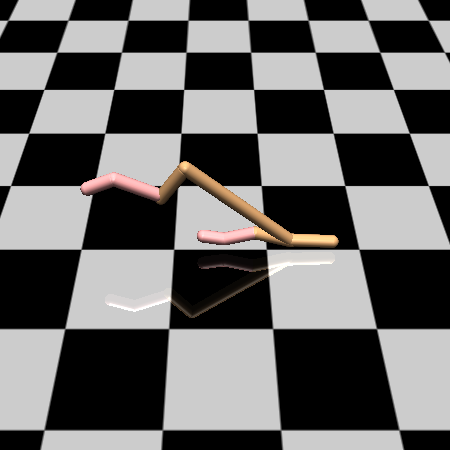}
    \includegraphics[width=0.13\textwidth]{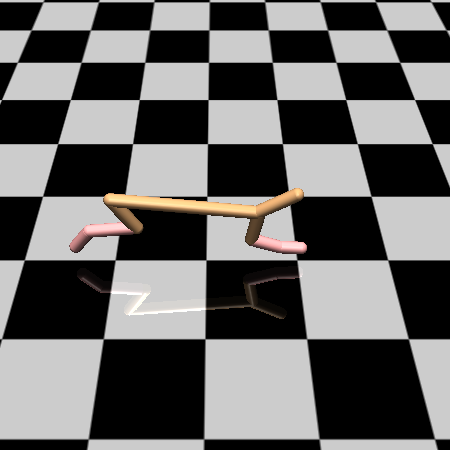}
    \includegraphics[width=0.13\textwidth]{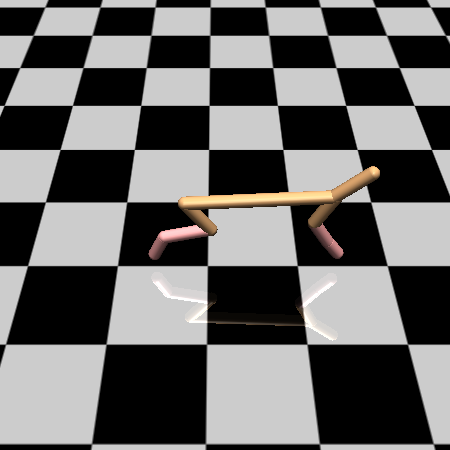}
    \includegraphics[width=0.13\textwidth]{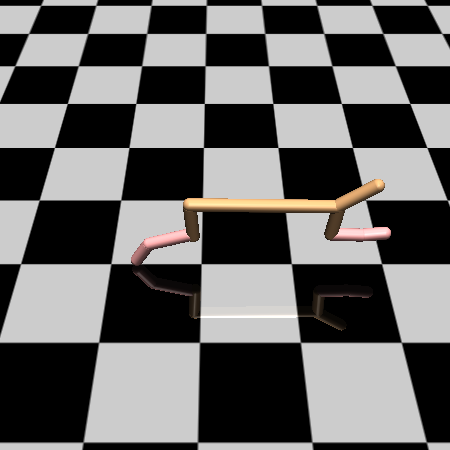}
    \includegraphics[width=0.13\textwidth]{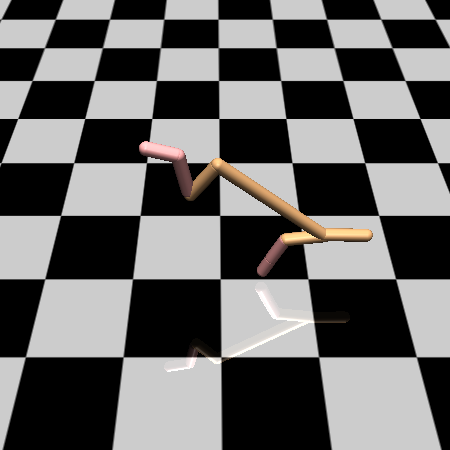}
    \caption{Agent's behavior under attacked dynamics using RL}
    \end{subfigure}
    \begin{subfigure}{\linewidth}
    \centering
    \includegraphics[width=0.13\textwidth]{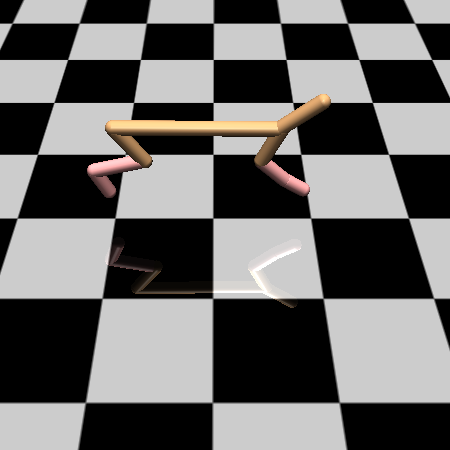}
    \includegraphics[width=0.13\textwidth]{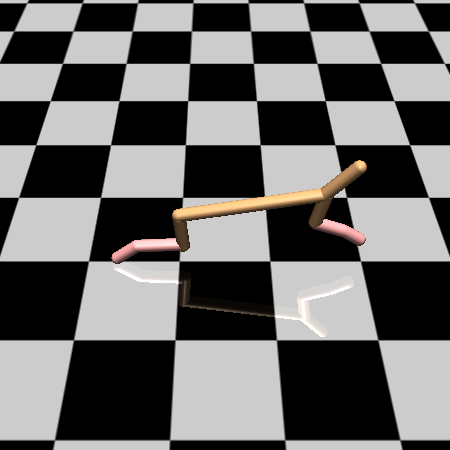}
    \includegraphics[width=0.13\textwidth]{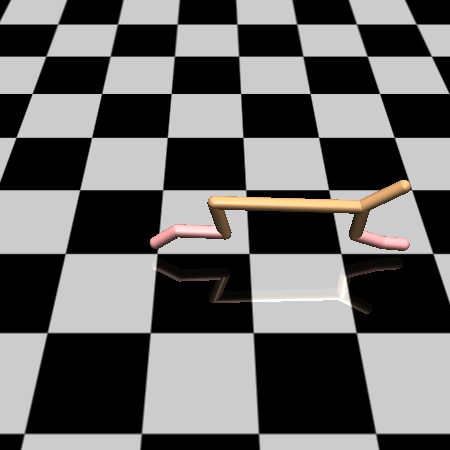}
    \includegraphics[width=0.13\textwidth]{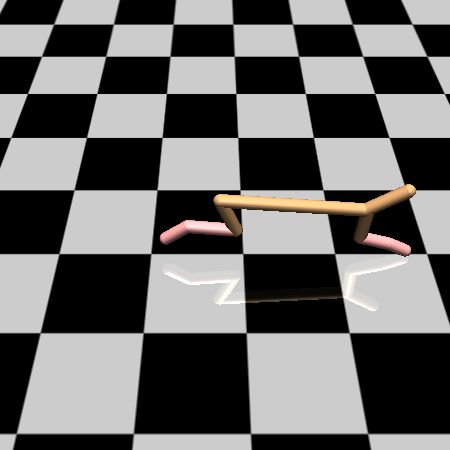}
    \includegraphics[width=0.13\textwidth]{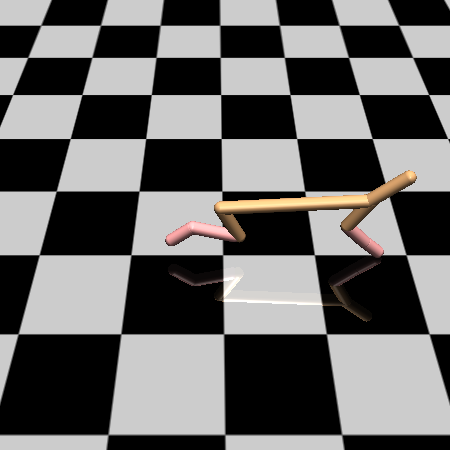}
    \includegraphics[width=0.13\textwidth]{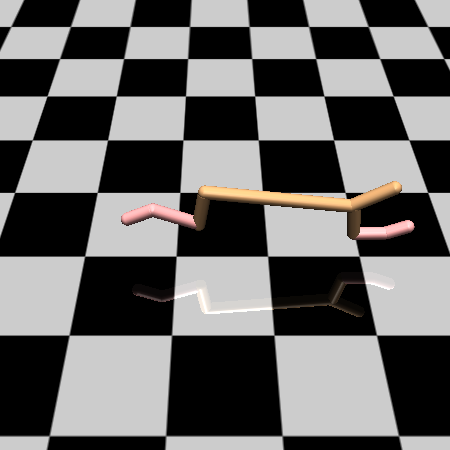}
    \caption{Agent's behavior under attacked dynamics using random search}
    \end{subfigure}
    \caption{Results for Dynamics Attack on HalfCheetah}
    \label{fig:env-half}
\end{figure}

\begin{figure}[h!]
\label{hopper}
    \centering
    \begin{subfigure}{\linewidth}
    \centering
    \includegraphics[width=0.12\textwidth]{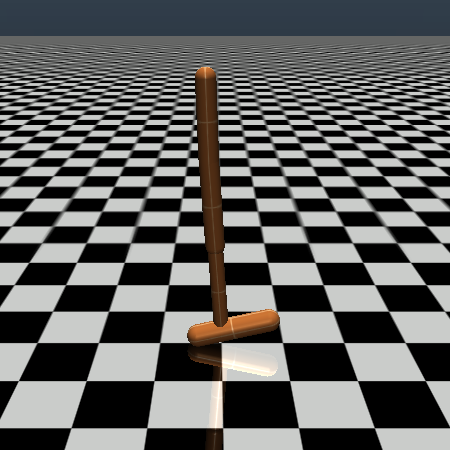}
    \includegraphics[width=0.12\textwidth]{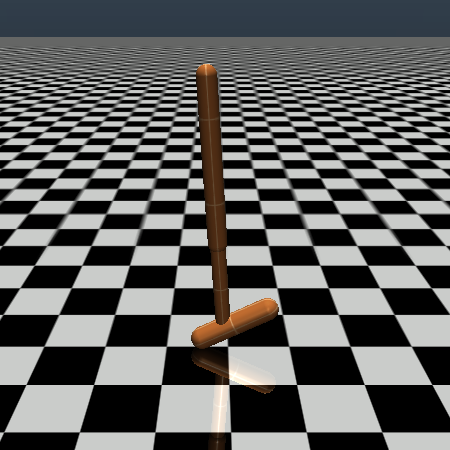}
    \includegraphics[width=0.12\textwidth]{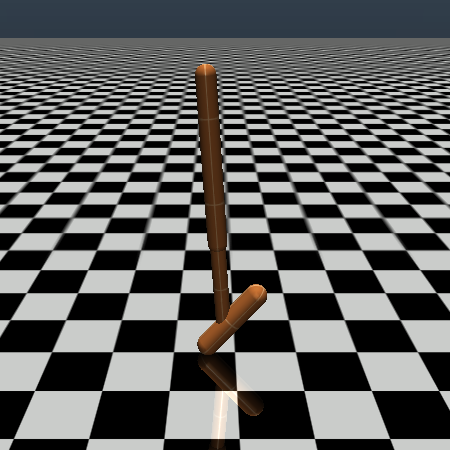}
    \includegraphics[width=0.12\textwidth]{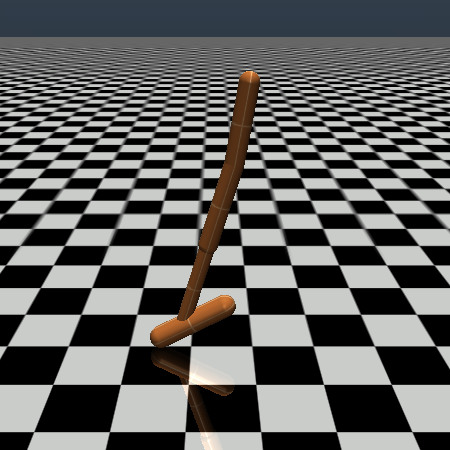}
    \includegraphics[width=0.12\textwidth]{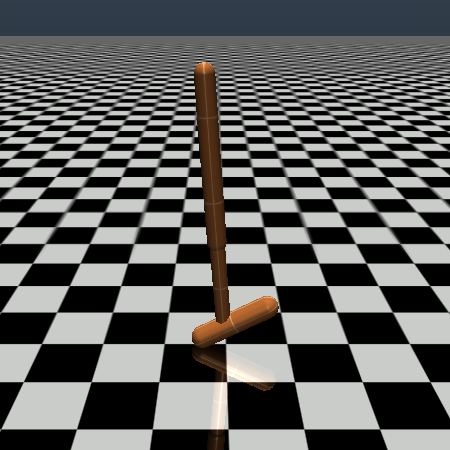}
    \includegraphics[width=0.12\textwidth]{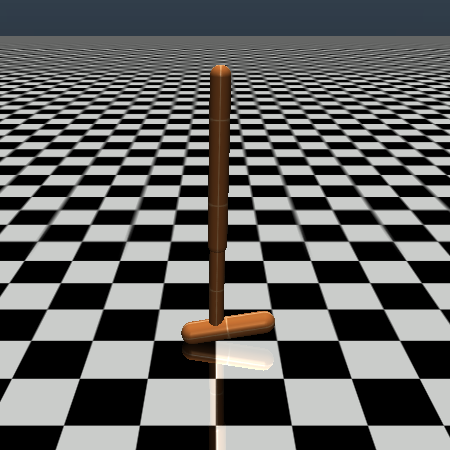}
    \includegraphics[width=0.12\textwidth]{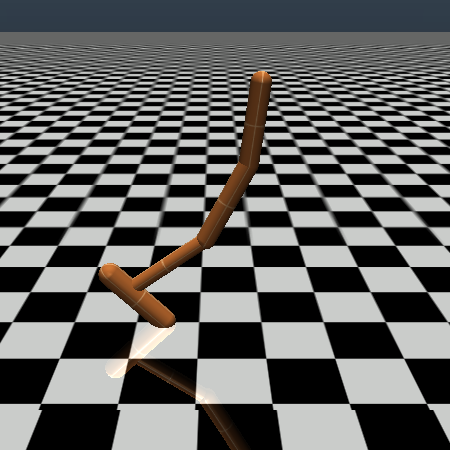}
    \caption{Agent's behavior under normal dynamics}
    \end{subfigure}
    \begin{subfigure}{\linewidth}
    \centering
    \includegraphics[width=0.12\textwidth]{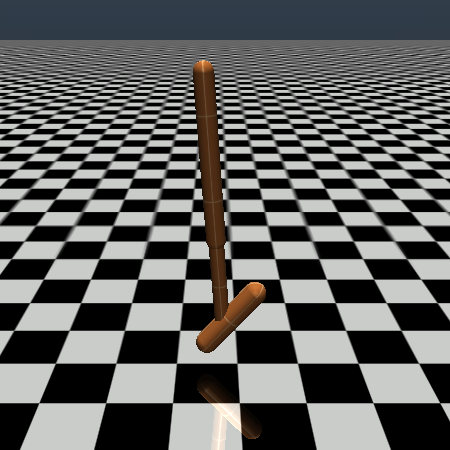}
    \includegraphics[width=0.12\textwidth]{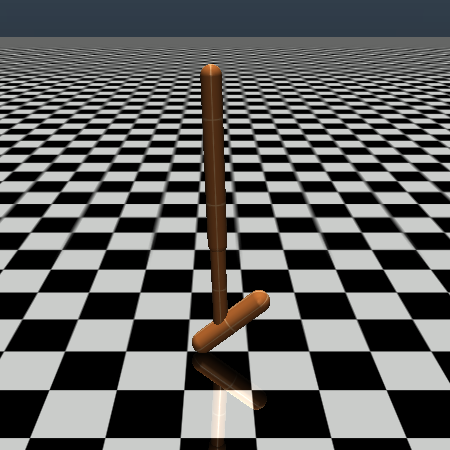}
    \includegraphics[width=0.12\textwidth]{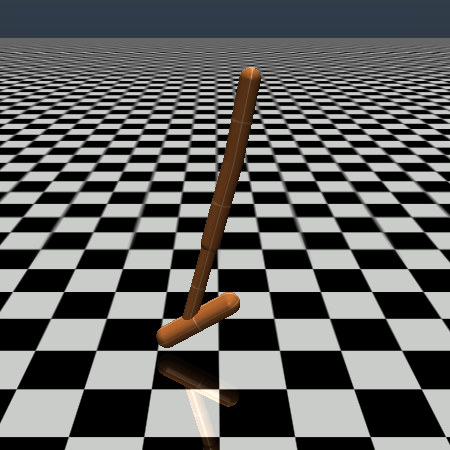}
    \includegraphics[width=0.12\textwidth]{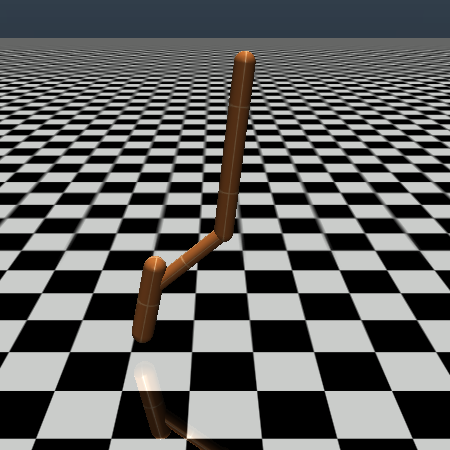}
    \includegraphics[width=0.12\textwidth]{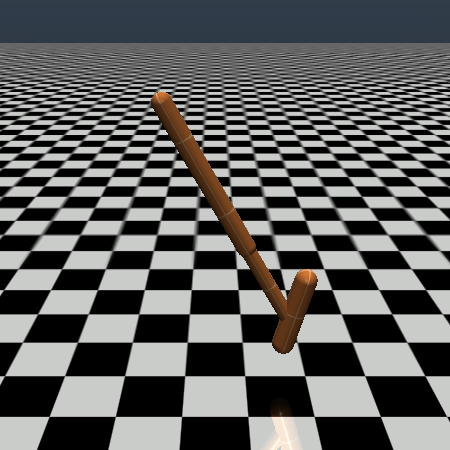}
    \includegraphics[width=0.12\textwidth]{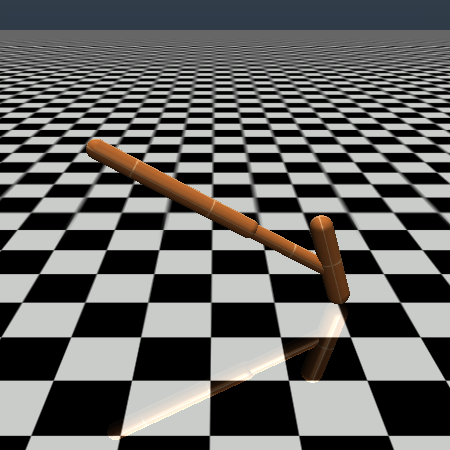}
    \includegraphics[width=0.12\textwidth]{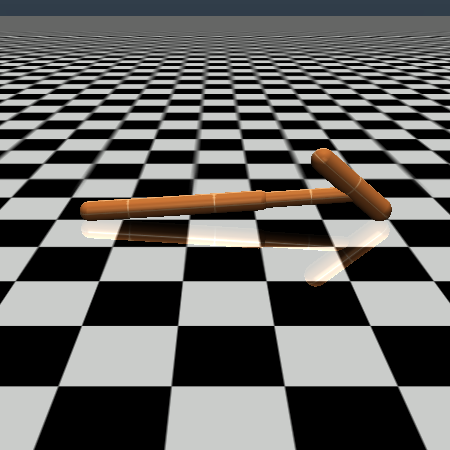}
    \caption{Agent's behavior under abnormal dynamics}
    \end{subfigure}
    \begin{subfigure}{\linewidth}
    \centering
    \includegraphics[width=0.12\textwidth]{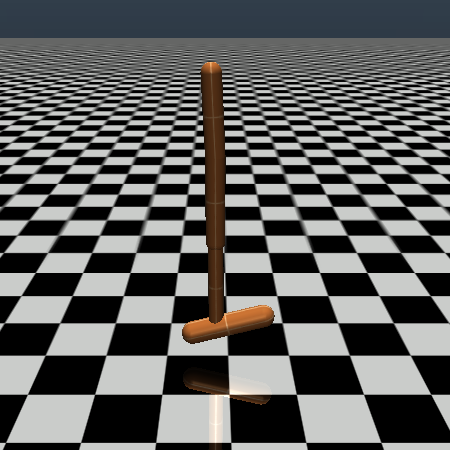}
    \includegraphics[width=0.12\textwidth]{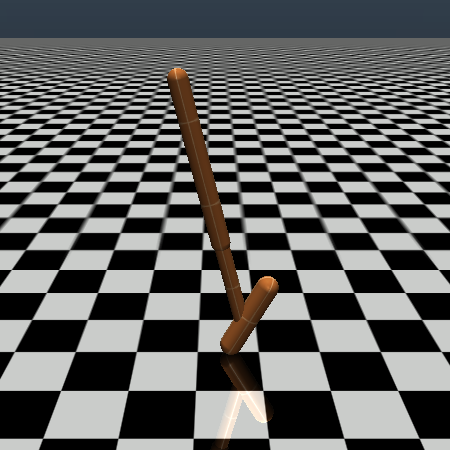}
    \includegraphics[width=0.12\textwidth]{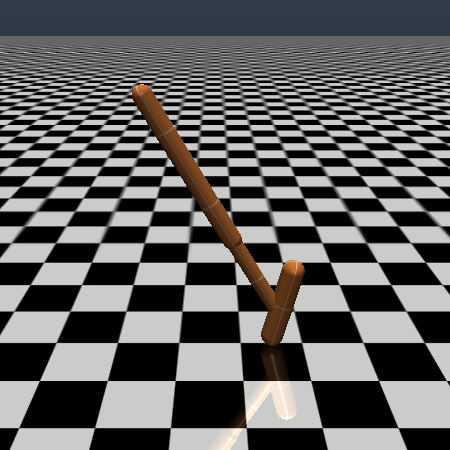}
    \includegraphics[width=0.12\textwidth]{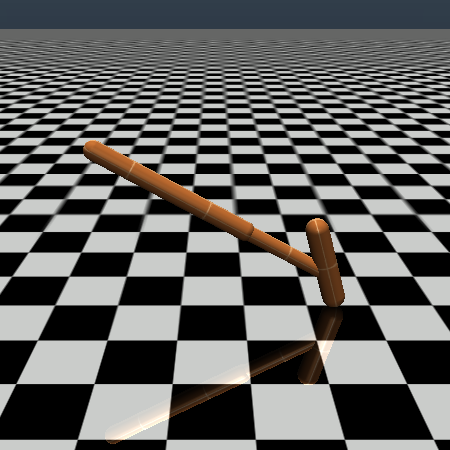}
    \includegraphics[width=0.12\textwidth]{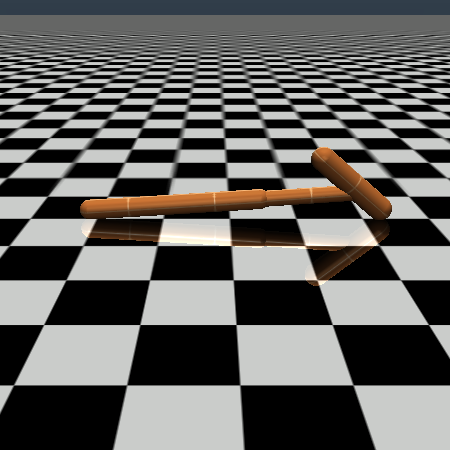}
    \includegraphics[width=0.12\textwidth]{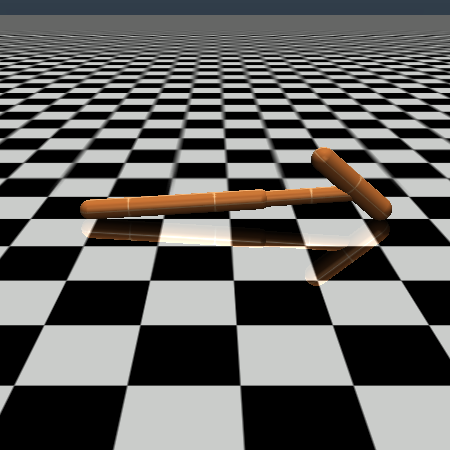}
    \includegraphics[width=0.12\textwidth]{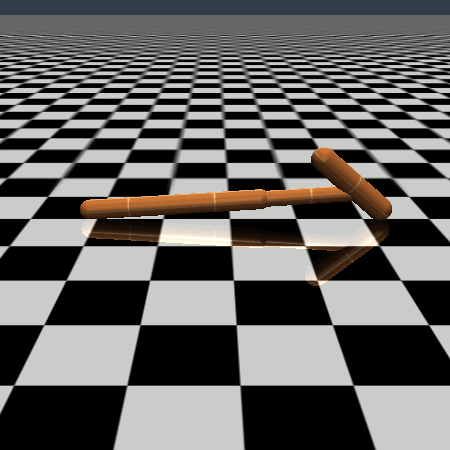}
    \caption{Agent's behavior under attacked dynamics using RL}
    \end{subfigure}
    \begin{subfigure}{\linewidth}
    \centering
    \includegraphics[width=0.12\textwidth]{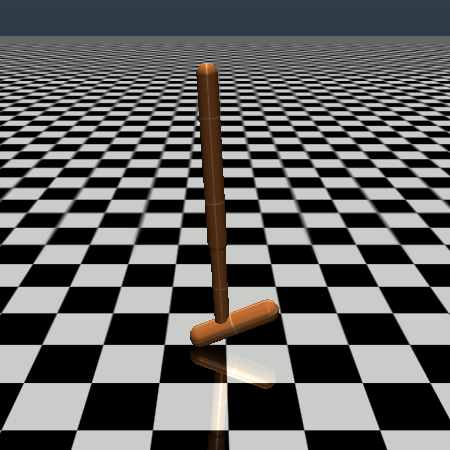}
    \includegraphics[width=0.12\textwidth]{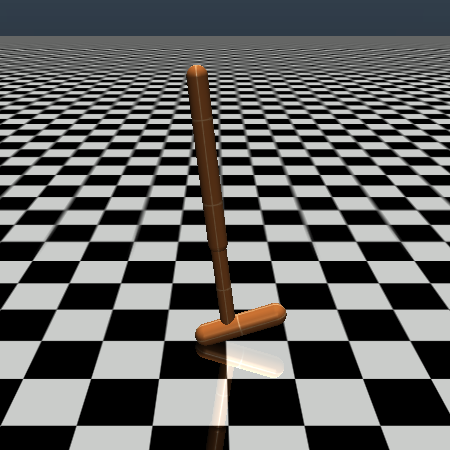}
    \includegraphics[width=0.12\textwidth]{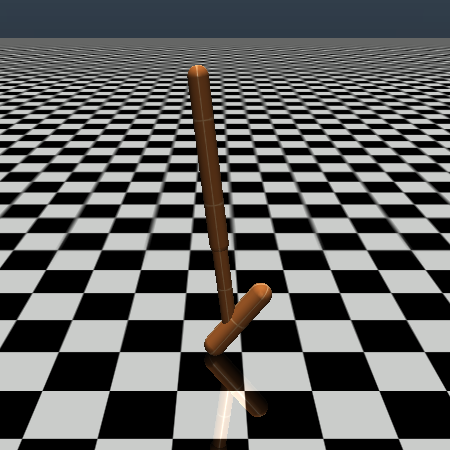}
    \includegraphics[width=0.12\textwidth]{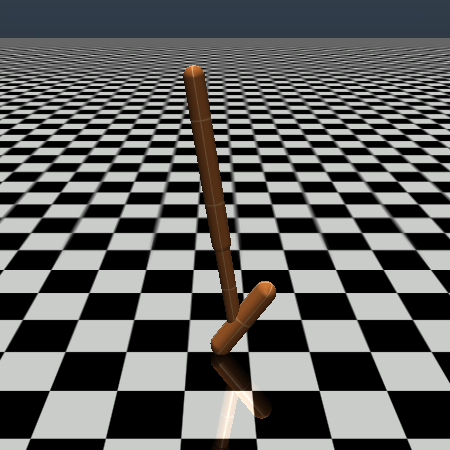}
     \includegraphics[width=0.12\textwidth]{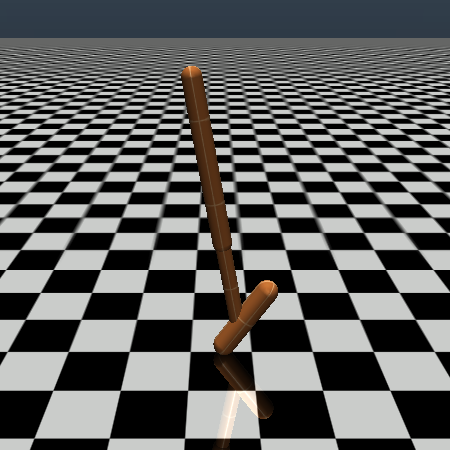}
    \includegraphics[width=0.12\textwidth]{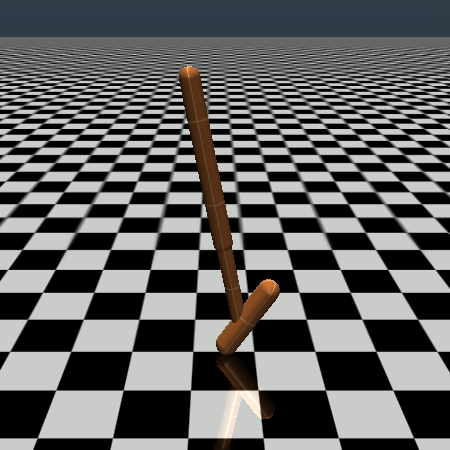}
    \includegraphics[width=0.12\textwidth]{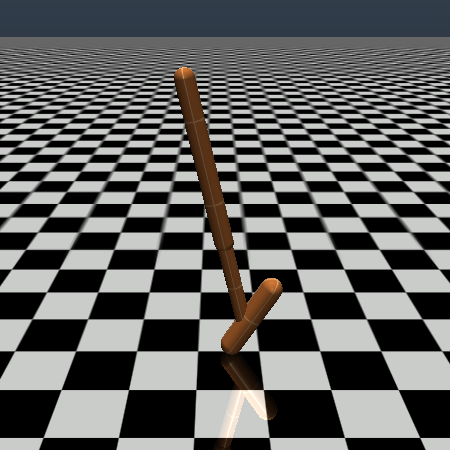}
    \caption{Agent's behavior under attacked dynamics using random search}
    \end{subfigure}
    \caption{Results for Dynamics Attack on Hopper}
    \label{fig:env-hopper}
\end{figure}

\begin{figure*}
    \centering
    \includegraphics[width=0.14\linewidth, height=0.14\linewidth]{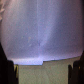}
    \includegraphics[width=0.14\linewidth, height=0.14\linewidth]{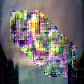}
    \includegraphics[width=0.14\linewidth, height=0.14\linewidth]{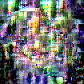}
    \includegraphics[width=0.14\linewidth, height=0.14\linewidth]{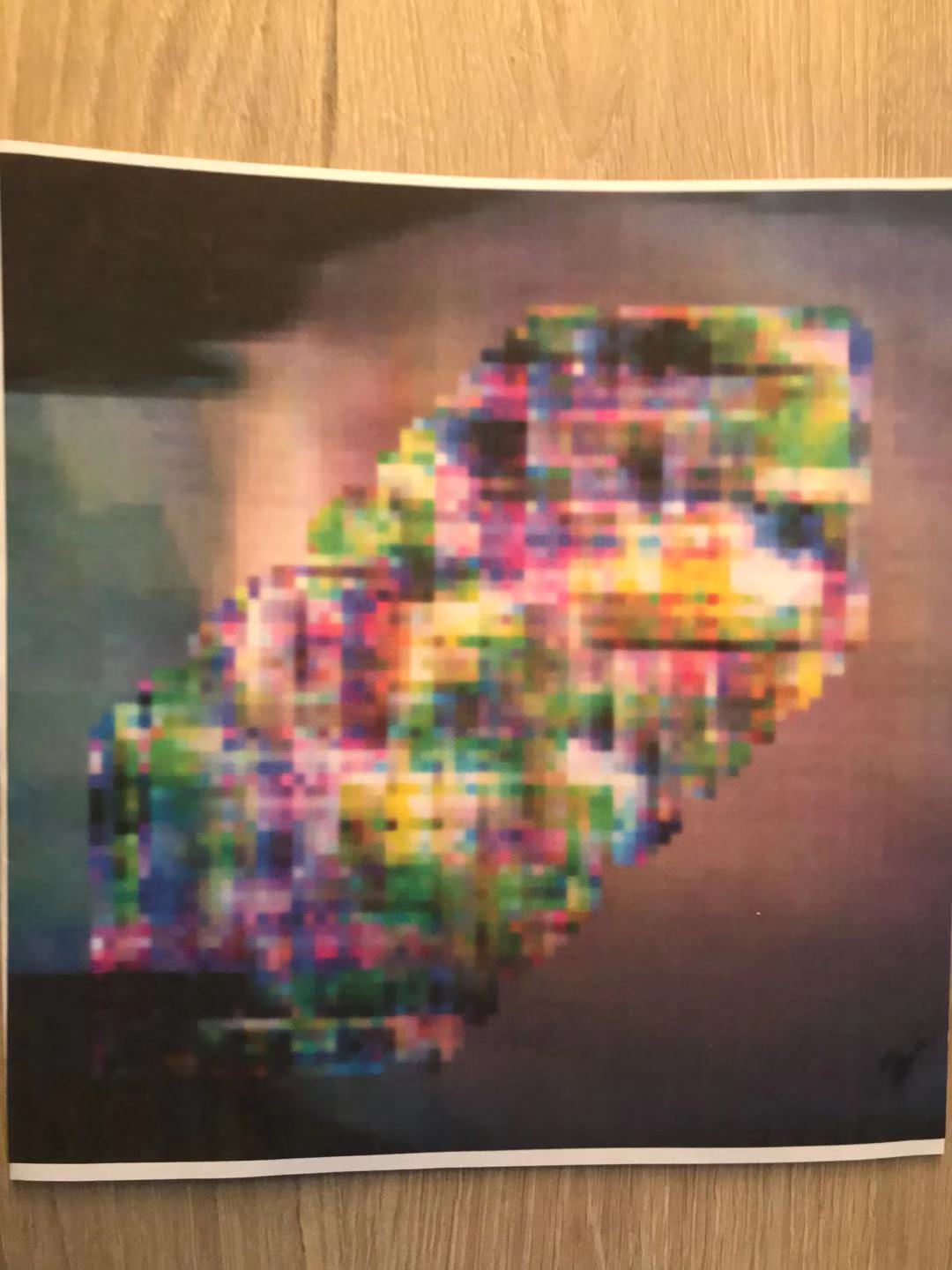}
    \includegraphics[width=0.14\linewidth, height=0.14\linewidth]{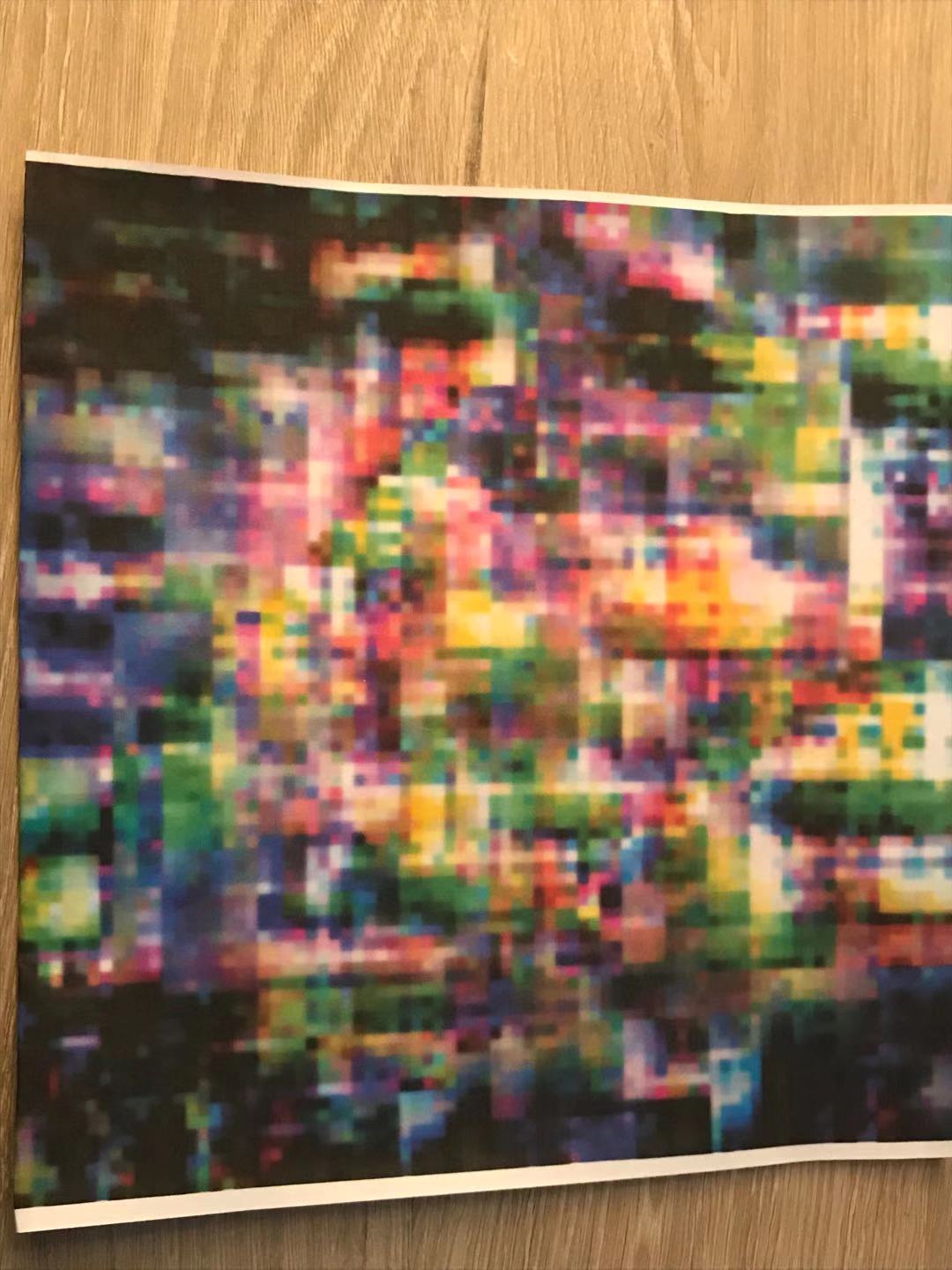}
    \includegraphics[width=0.14\linewidth, height=0.14\linewidth]{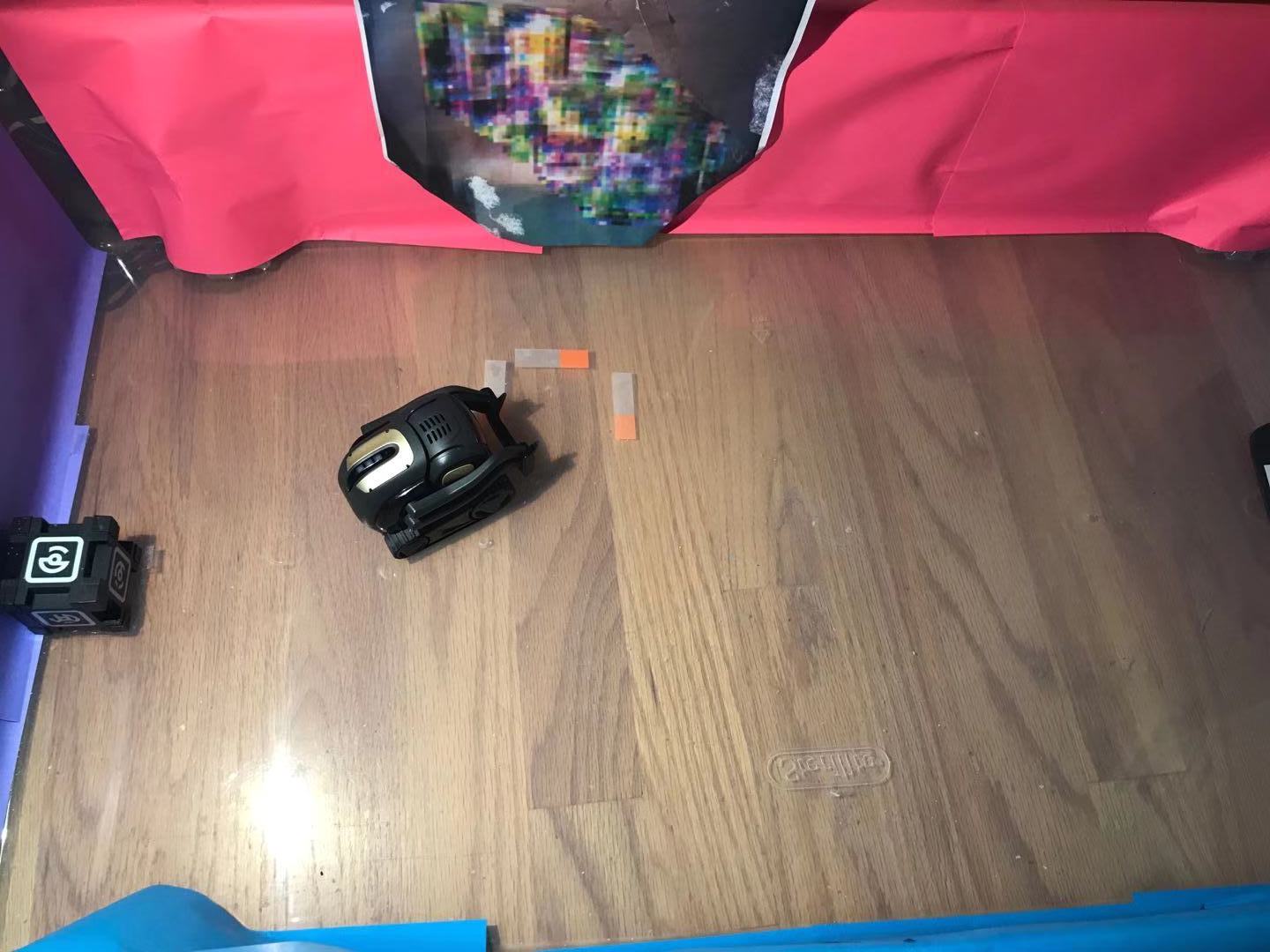}
    \caption{Original robot observation (first), optimized perturbed image with white-box approach (second), optimized perturbed image with black-box approach (third), printed perturbation patch with white-box approach (forth), printed perturbation patch with black-box approach (fifth) physical environment with the playground, adversarial patch and the robot (sixth). }
    \label{fig:perturbed_image}
\end{figure*}

\section{Experimental Setup}
\label{sec:expsetup}
We trained DQN models on Pong, Enduro, and TORCS,
and trained DDPG models on HalfCheetah and Hopper.
The DQN model for training Pong and Enduro consists of 3 
convolutional layers and 2 fully connected layers. 
The two network architectures differ in their number of 
filters. Specifically, the first network structure is
$C(4,32,8,4)-C(32,64,4,2)-C(64,64,3,1)-F(3136,512)-F(512,na)$,
where $C(c_1, c_2, k, s)$ denotes a convolutional layer
of input channel number $c_1$, output channel number $c_2$,
kernel size $k$, and stride $s$. $F(h_1, h_2)$ denotes a
fully connected layer with input dimension $h_1$ and output 
dimension $h_2$, and $na$ is the number of actions in
the environment. The DQN model for training TORCS consists
of 3 convultional layers and 2 or 3 fully connected layers.
The convultional layers' structure is $C(12,32,8,4)-C(32,64,4,2)-
C(64,64,3,1)$, and the fully connected layer structure is
$F(3136, 512)-F(512, 9)$ for one model and $F(3136,512)-F(512,128)-F(128,9)$
for the other model. 

The DDPG model for training HalfCheetah and Hopper 
consists of several fully connected layers. We trained 
two different policy network structures on all MuJoCo environments.
The first model's actor is a network of size $F(dim_{in},64)-F(
64, 64)-F(64, na)$ and the critic is a network of size
$F(dim_{in}, 64)-F(64, 64)-F(64, 1)$. The second model's
actor is a network of size $F(dim_{in}, 64)-F(64,64)-F(64,64)-F(64,64)-F(64, na)$, and the critic
is a network of size $F(dim_{in}, 64)-F(64, 64)-F(64, 64)-
F(64, 64)-F(64, 1)$. For both models,
we added ReLU activation layers between these 
fully connected layers.  

For the network structure of \textsf{obs-nn-wb} attack, we use C(12,8,7,1)-C(8,16, 3, 2)-C(16, 32, 3, 2) -R(32,32,3,1)- R(32,32,3,1)-R(32,32,3,1)- R(32,32,3,1)-R(32,32,3,1)-R(32,32,3,1)-C(32,16,3,2)--C(16,8,3,2)-C(8,12, 7,7) for torcs and C(4,8,7,1)-C(8,16, 3, 2)-C(16, 32, 3, 2) -R(32,32,3,1)- R(32,32,3,1)-R(32,32,3,1)- R(32,32,3,1)-R(32,32,3,1)- R(32,32,3,1)-C(32,16,3,2)--C(16,8,3,2)-C(8,4, 7,7) for atari games. R(c1,c2,k,s) denotes a residual block layer of input channl c1, output channel c2, kernel size k, and stride s. 

The TORCS autonomous driving environment is a discrete
action space control environment with 9 actions, they 
are turn left, turn right, keep going, turn left and accelerate,
turn right and accelerate, accelerate, turn left and decelerate,
turn right and decelerate and decelerate. The other
4 games, Pong, Enduro, HalfCheetah, and Hopper are 
standard OpenAI gym environment. 

The trained model's performance when tested
without any attack is included in the following
Table~\ref{tab:tab1}.

\begin{table}[h]
    \centering
    \caption{ Model performance among different environments}
    \label{tab:tab1}
    \begin{tabular}{cccccc}
         \toprule
        & Torcs & Enduro & Pong & HalfCheetah & Hopper \\ 
        \midrule
        Reward & 1720.8 & 1308 & 21 & 8257 & 3061  \\ 
        \# of steps &1351 &16634 & 1654 &1000&1000\\
         \bottomrule
    \end{tabular}
\end{table}

The DDPG neural network used for \textsf{env-search-bb} is the same as the first model (3-layer fully connected network) used for training the policy for HalfCheetah, except that the input dimension $dim_{in}$ is of the perturbation
parameters' dimension, and output dimension is also of the perturbation parameters' dimension. For HalfCheetah, Hopper and TORCS, these input and output dimensions
are 32, 20, and 10, respectively.

\section{Additional Experiment Results}

\paragraph{Example Trajectory of Agents Under Dynamics Attack}
In Figure~\ref{fig:env-half}, Figure~\ref{fig:env-hopper},
we show the sequences of states when the agents
are under attack with the random search or reinforcement
learning based search method. 
The last image in each 
sequence denotes the state at same step $t$. The
last image in each abnormal dynamics rollout sequence
corresponds to the target state, the last image
in the attacked dynamics using RL search denotes the
attacked results using \textsf{env-search-bb}, and the last image in the attacked dynamics using random search denotes the attacked results using \textsf{env-rand-bb}. It can be seen
from these figures that \textsf{env-search-bb} method is very effective
at achieving targeted attack while using random search,
it is relatively harder to achieve this.

\paragraph{Real world attack additional results} We include in Figure~\ref{fig:perturbed_image} the original robot observation, the optimized perturbed image, the actual printed perturbed image patch and the bird-view of the robot environment after the patch has been mounted. 


\end{document}